%% file: main.tex
\pdfoutput=1
\documentclass{article}

\usepackage{microtype}
\usepackage{graphicx}
\usepackage{booktabs}
\usepackage{hyperref}

\newif\ifanonymous
\anonymousfalse              % arXiv: author version
\ifdefined\AUTHORS\anonymousfalse\fi
\ifanonymous
  \usepackage{icml2027}
\else
  \usepackage[preprint]{icml2027}
  \makeatletter\renewcommand{\ICML@preprint}{}\makeatother % no "Preprint. <date>." footer line
\fi

\input{math_commands}

\icmltitlerunning{Depth Enables Local Entropy}

\begin{document}

\twocolumn[
  \icmltitle{Depth Enables Local Entropy: Quadratic Depth Dependence\\
    in Deep Variation-Norm ReLU Regression}

  \icmlsetsymbol{equal}{*}

  \begin{icmlauthorlist}
    \icmlauthor{Tao Jiang}{equal,iscas,ucas}
    \icmlauthor{Minbo Gao}{equal,iscas,ucas}
    \icmlauthor{Shaowei Cai}{iscas,ucas}
  \end{icmlauthorlist}

  \icmlaffiliation{iscas}{Key Laboratory of System Software (Chinese Academy of Sciences) and State Key Laboratory of Computer Science, Institute of Software, Chinese Academy of Sciences, Beijing, China}
  \icmlaffiliation{ucas}{School of Computer Science and Technology, University of Chinese Academy of Sciences, Beijing, China}

  \icmlcorrespondingauthor{Tao Jiang}{jiangt@ios.ac.cn}
  \icmlcorrespondingauthor{Minbo Gao}{gaomb@ios.ac.cn}
  \icmlcorrespondingauthor{Shaowei Cai}{caisw@ios.ac.cn}

  \icmlkeywords{minimax risk, deep ReLU networks, variation spaces, metric entropy, local packing, bit extraction, second-order total variation}

  \vskip 0.3in
]

\printAffiliationsAndNotice{\icmlEqualContribution}

\begin{abstract}
How small can a layer-sum variation budget be while deep ReLU regression keeps a statistical complexity quadratic in depth? We study Gaussian regression on the circle over compositions of $L$ vector-valued ReLU blocks of width $w$, whose second-order variation costs sum to at most $A$. With unit output bound and noise level, a budget $CLw^\varepsilon$ already forces minimax risk $\Omega(\varepsilon^2\min\{1,L^2w^2\log w/n\})$ for each fixed $\varepsilon\in(0,2]$ and $L\gtrsim1/\varepsilon$. This is far below the budget of order $Lw^2$ needed when every layer carries the full gain. When $L\gtrsim\log w$, a budget $CL$ forces risk $\Omega(\min\{1,L^2w^2/(n\log w)\})$, and over a wide range of sample sizes no budget below a small multiple of $L$ reaches this risk. Conversely, every function in the class has derivative variation at most $2(A/L)^L$, independently of the width, and in a low-budget regime the minimax risk matches that of a one-dimensional second-variation ball of radius $2(A/L)^L$. The budget thus moves the class from the behavior of a one-dimensional smoothness ball to statistical complexity quadratic in depth, and a budget linear in depth already suffices for the latter.
\end{abstract}

\input{sections/01_introduction}
\input{sections/02_setup}
\input{sections/03_results}
\input{sections/04_proofs}
\input{sections/05_discussion}

\section*{Impact Statement}
This paper presents theoretical work whose goal is to advance the understanding of statistical estimation with deep neural networks. It contains no experiments, data, or deployed systems. There are many potential societal consequences of progress in machine learning, none of which we feel must be specifically highlighted here.

\section*{AI Disclosure Statement}
Generative AI tools were used to assist with literature exploration, the development and refinement of mathematical arguments, and manuscript preparation. This included assistance with mathematical derivations, proof checking, and improving the organization and clarity of the presentation. The authors reviewed the AI-assisted content and take full responsibility for the final manuscript.

\bibliography{references}
\bibliographystyle{icml2027}

\newpage
\appendix
\onecolumn
\input{sections/A_architecture}
\input{sections/B_approximation}
\input{sections/C_grid_code}
\input{sections/D_packing}
\input{sections/E_lower_bounds}
\input{sections/F_variation}
\input{sections/G_upper_bound}
\input{sections/H_low_budget}

\end{document}

%% file: math_commands.tex
\usepackage{amsmath,amssymb,mathtools}
\usepackage{amsthm}
\usepackage{booktabs,array}
\usepackage{enumitem}
\usepackage{tikz}
\usetikzlibrary{arrows.meta,positioning,calc,decorations.pathreplacing}
\definecolor{accentblue}{RGB}{38,84,150}

\newcommand{\R}{\mathbb{R}}

\newcommand{\E}{\mathbb{E}}
\newcommand{\cF}{\mathcal{F}}
\newcommand{\cC}{\mathcal{C}}
\newcommand{\cG}{\mathcal{G}}
\newcommand{\cZ}{\mathcal{Z}}
\newcommand{\cN}{\mathcal{N}}
\newcommand{\Risk}{\mathfrak{R}}
\newcommand{\VarCost}{\mathfrak{V}}
\newcommand{\RBV}{\mathcal{R}\mathrm{BV}^{2}}
\newcommand{\relu}{\rho}
\newcommand{\Lip}{\operatorname{Lip}}
\newcommand{\Pdim}{\operatorname{Pdim}}
\newcommand{\KL}{D_{\mathrm{KL}}}
\newcommand{\eps}{\varepsilon}
\newcommand{\wtO}{\widetilde O}
\newcommand{\pos}[1]{(#1)_{+}}
\newcommand{\norm}[1]{\left\lVert #1\right\rVert}

\newcommand{\dd}{\,\mathrm{d}}
\newcommand{\defeq}{\coloneqq}

\newcommand{\Capp}{C_{\mathrm{app}}}

\theoremstyle{plain}
\newtheorem{theorem}{Theorem}[section]
\newtheorem{corollary}[theorem]{Corollary}
\newtheorem{proposition}[theorem]{Proposition}
\newtheorem{lemma}[theorem]{Lemma}
\theoremstyle{definition}

\theoremstyle{remark}
\newtheorem{remark}[theorem]{Remark}

\newcommand{\TV}{\operatorname{TV}}
\newcommand{\Var}{\operatorname{Var}}
\newcommand{\cH}{\mathcal{H}}

\newcommand{\Dsig}{\sigma_{\mathrm{eff}}^2}

%% file: sections/01_introduction.tex
\section{Introduction}
\label{sec:intro}

Depth affects the statistical complexity of a network class in two different ways. For classes defined by their architecture alone, the relevant quantity is the pseudodimension. A piecewise-linear network with $W$ parameters and $L$ layers has pseudodimension $O(WL\log W)$, and bit-extraction constructions show that this is tight up to the logarithm \citep{bartlett2019nearly}. For networks of width $w$, depth thus contributes $L^2w^2$ rather than the parameter count $Lw^2$. These constructions, however, use weights that grow with depth, whereas analyses of overparameterized networks control the size of the weights rather than their number \citep{neyshabur2015norm,golowich2018size,parkinson2024depth}. Under a size constraint, a large pseudodimension says little about statistical difficulty. The minimax risk at sample size $n$ and noise level $\sigma$ is governed by local entropy: how many functions of the class are pairwise separated at the testing scale $\sigma\sqrt{M/n}$, where $M$ is the logarithm of their number. A class may contain its combinatorially rich functions only at amplitudes far below this scale. Whether depth increases statistical complexity is therefore a question about local entropy under a size constraint.

In the variation-space framework, a natural size constraint is a layer-sum budget \citep{parhi2022kinds,shenouda2024variation}. Each layer is a shallow ReLU block, measured by its second-order variation in the Radon domain, and the costs of the layers are added. We study the class $\cC_{L,w}(A,B)$ of functions on $[0,2)$, a parameterization of the circle, that are compositions of $L$ vector-valued blocks with at most $w$ ReLU atoms and intermediate dimensions at most $w$, with total cost at most $A$ and sup norm at most $B$ (Section~\ref{sec:setting}). The architecture has $O(Lw^2)$ parameters. For $L$ and $w$ above a universal threshold, the pseudodimension of the class is of order $L^2w^2$ up to logarithms at every budget (Proposition~\ref{prop:fat}). The cost is attached to representations and is not a norm on the composed function. For instance, $t\mapsto\gamma t$ is the composition of $L$ blocks $u\mapsto\gamma^{1/L}u$ of total cost $L\gamma^{1/L}$, so gains can be spread over layers, while the cost of each layer grows with the size of its coefficients. We ask:
\begin{quote}
\emph{How does the layer-sum budget $A$ control the minimax risk of regression over $\cC_{L,w}(A,B)$? How small can $A$ be while the risk keeps its quadratic dependence on depth?}
\end{quote}

Existing results do not settle this question. The pseudodimension yields an upper bound of order $(\sigma^2+B^2)L^2w^2\log(Lw)\log(n)/n$ at every budget. Since the pseudodimension stays of order $L^2w^2$ at every budget, no argument based on it alone can do better. A lower bound on the pseudodimension, however, is not a minimax lower bound. Fano's inequality needs $\exp(\Omega(M))$ functions that are pairwise separated at the scale $\sigma\sqrt{M/n}$. Bit-extraction networks with unit coefficients have amplitude of order $1/M$, and amplifying them costs budget. \citet{ganguli2026dichotomy} study a closely related class with $A=B=R$ and $\sigma\asymp R$. They show that the minimax risk lies between $\Omega(Lw^2R^2/n)$ and $\wtO(L^2w^2R^2/n)$. They also observe that the bit-extraction networks behind the pseudodimension lower bound have weights growing with depth and fall outside their norm ball, and they conjecture the depth-linear rate.\footnote{Their displayed block norm is scalar and the intermediate dimensions are not specified; the vector-valued class studied here matches their $O(Lw^2)$ parameter count. Compositions of scalar-to-scalar blocks have $O(Lw)$ parameters and form a different problem, which we do not consider.} For the vector-valued class, the available lower bound is therefore linear in depth, and the available upper bound does not depend on the budget. A priori the answer could go either way. Additive costs let gains compound across depth \citep{nakhleh2026deep}, which suggests that amplification is cheap. But a layer of width $w$ pays of order $w^2$ per unit of gain, and amplifying every layer of a bit-extraction network requires a budget of order $Lw^2$.

\paragraph{Our contributions.}
We answer the question at both ends of the budget range, and the two answers meet at a single risk level, where we determine the critical budget.

Our first result is that quadratic depth dependence survives budgets far below $Lw^2$. For every positive budget and every coding depth $d$ with $D_0\le d\le L-1$, where $D_0$ is a universal constant, the class contains a local packing of $\exp(\Omega(d^2w^2\log w))$ functions. Their $L^2(\mu)$ norms are $O(\lambda)$ and their pairwise distances are $\Theta(\lambda)$, for every amplitude $\lambda$ up to $\min\{c\,B,\Lambda_d(A)\}$, where $\Lambda_d(A)$ is an explicit scale (Theorem~\ref{thm:packing}). Fano's inequality turns these packings into minimax lower bounds (Theorem~\ref{thm:general-lower}). With $B=\sigma=1$, a budget of order $Lw^2$ gives risk $\Omega(\min\{1,L^2w^2\log w/n\})$, which matches the pseudodimension upper bound up to a factor $\log(Lw)\log n/\log w$. For each fixed $\varepsilon\in(0,2]$ and $L\gtrsim1/\varepsilon$, a budget $CLw^\varepsilon$ already gives risk $\Omega(\varepsilon^2\min\{1,L^2w^2\log w/n\})$. When $L\gtrsim\log w$, a budget $CL$ gives risk $\Omega(\min\{1,L^2w^2/(n\log w)\})$ (Corollary~\ref{cor:budgets}). These budgets are sufficient for our construction. Only the last is shown to be necessary, and only for its own risk level. In the normalization of \citet{ganguli2026dichotomy}, and for $n\gtrsim L^2w^2\log w$, the same packings give risk at least of order $R^2L^2w^2\log w/n$ for the vector-valued class once $R$ exceeds an explicit polynomial in $L$ and $w$. At such radii the risk of this class is therefore not linear in depth.

At the other end, a small budget forces every function in the class to be smooth in a one-dimensional sense. Every $f\in\cC_{L,w}(A,B)$ has derivative variation $J(f)\le2(A/L)^L$, counting the boundary slopes, with no factor depending on the width (Proposition~\ref{prop:variation}). After centering, the class therefore lies in a one-dimensional second-variation ball of radius $2\kappa$, where $\kappa=(A/L)^L$, and this ball has metric entropy of order $(\kappa/\eps)^{1/2}$ at scale $\eps$. The exact cost of constant functions controls the mean: the largest constant in the class is $H_L(A)=\max_{1\le j\le L}(A/j)^j$. Together these facts give an upper bound that depends on the budget but not on the width (Theorem~\ref{thm:upper}). Spline hypercubes, realized in one block and amplified by scalar gains, show that this bound is attained. Suppose the optimal number of tents fits into one block and the centered functions are at most of the order of the noise level. Then Theorem~\ref{thm:low} gives
\[
  \Risk_n^*(A,B,\sigma)\asymp\min\big\{B^2,\ H_L(A)^2,\ q+\kappa^{2/5}q^{4/5}\big\},
\]
where $q=\sigma^2/n$. The exponent $4/5$ is the classical rate for second-order total variation \citep{mammen1997locally}. What we add is the matching characterization for the deep, finite-width class with the anchored cost.

The two bounds also determine how much budget is needed for one particular risk level, namely $L^2w^2/(n\log w)$. Let $B=\sigma=1$, $w$ large, $L\gtrsim\log w$, and $L^2w^2/\log w\le n\le(L^2w^2/\log w)^5$. The packings show that a budget $CL$ reaches risk $c_*L^2w^2/(n\log w)$. The budget-sensitive upper bound shows that every budget below $c'L$ gives smaller risk. The intrinsic critical budget for this risk level is thus $\Theta(L)$ (Proposition~\ref{prop:critical}). The result says nothing about the larger level $L^2w^2\log w/n$, for which our sufficient budget exceeds the necessary one by a factor polynomial in $w$. Between the two ends of the budget range, and in particular across the transition, the minimax rate is not determined.

\paragraph{Main proof ideas.}
The lower bounds rest on bounded-coefficient bit extraction \citep{ou2024quantization}. It encodes $M_d\asymp d^2w^2\log w$ bits in a depth-$d$ network of width $w$, but only at amplitude about $1/M_d$, and the amplification must be paid for in budget. A block of width $w$ pays of order $w^2$ per unit of gain, whereas a scalar block $u\mapsto au$ pays only $a$, and gains multiply along the composition. We therefore give the coding network only $d$ of the $L$ blocks and spend the remaining blocks on scalar gains. All gains are balanced so that their product is the required amplification (Figure~\ref{fig:allocation}). The resulting cost is at most $2L(\lambda M_d)^{1/L}w^{2d/L}$. Shortening the code loses the factor $(d/L)^2$ in the number of bits, but lowers the width factor from $w^2$ to $w^{2d/L}$.

The upper bounds rest on a comparison between the cost of a block and the derivative variation of its output along a path. Each ridge is a maximum of two affine branches with complementary coefficients, and the affine part of the block is expressed through its anchor terms. Hence a block of cost $a$ multiplies the vector of derivative variations by at most $a$, in Euclidean norm and with no width factor (Lemma~\ref{lem:anchored}). Iterating over the blocks and applying AM--GM gives $J(f)\le2(A/L)^L$. This is the step that links the deep class to classical second-variation regression.

\subsection{Related work and scope}
\label{sec:related}

Shallow variation spaces, their entropy, and minimax rates over them are studied by \citet{bach2017breaking,savarese2019infinite,ongie2020function,parhi2021banach,siegel2024sharp,parhi2023near}. One-dimensional second-variation regression goes back to \citet{mammen1997locally}. \citet{parhi2022kinds} introduced the deep variation spaces whose layer-sum cost we constrain, and \citet{shenouda2024variation} study vector-valued variation spaces. \citet{nakhleh2026deep} consider a recursively normalized deep variation norm and prove that univariate deep ReLU classes saturate: depth enlarges the unit ball by at most a factor two. They also note that additive layer-sum costs permit compounded rescaling. Our derivative-variation bound is the analogue for the anchored cost studied here, and our lower bounds show that, for this cost, compounded rescaling does produce statistical complexity quadratic in depth. Norm-based capacity bounds \citep{neyshabur2015norm,golowich2018size} and norm-controlled depth separations \citep{parkinson2024depth} constrain weight norms rather than layer-sum variation budgets.

Bit extraction underlies the VC and pseudodimension lower bounds of \citet{bartlett2019nearly} and the optimal approximation rates of very deep networks \citep{yarotsky2018optimal}. \citet{ou2024quantization} construct the unit-coefficient bit-extraction networks we use; we retain their decoder and approximation estimates, and handle the biases in their coefficient-rescaling steps exactly (Appendix~\ref{app:corrected-approx}). Absorbing biases into a constant coordinate is a standard device. \citet{ou2026covering} determine the global covering numbers of bounded-weight deep ReLU networks, whose logarithm is quadratic in depth. Our packings are instead localized at the statistical scale and respect a layer-sum budget. We use the entropy--minimax correspondence of \citet{birge1983approximation,yang1999information,tsybakov2009introduction}, and the intrinsic-variability argument of \citet{huster2018limitations} in the derivative-variation bound.

Our results concern estimation over a function class, not training algorithms. The minimax rate is determined in the low-budget regime and, up to logarithmic factors, in the high-complexity regime. Between the two regimes it is not determined (Section~\ref{sec:discussion}).

%% file: sections/02_setup.tex
\section{Problem setup}
\label{sec:setting}

\paragraph{Regression model.}
We parameterize the circle by $t\in[0,2)$, the coordinate $t=\theta/\pi$ of \citet{ganguli2026dichotomy}, and let $\mu(\dd t)=\dd t/2$ be the uniform probability measure on $[0,2)$. We observe $n$ independent pairs
\begin{equation}
  T_i\sim\mu,\qquad Y_i=f^\star(T_i)+\xi_i,\qquad \xi_i\sim N(0,\sigma^2),
  \label{eq:model}
\end{equation}
with $\xi_i$ independent of $T_i$. For a class $\cF$, the minimax risk is
\begin{equation}
  \Risk_n^*(\cF,\sigma)\defeq
  \inf_{\widehat f}\sup_{f^\star\in\cF}
  \E_{f^\star}\norm{\widehat f-f^\star}_{L^2(\mu)}^2,
  \label{eq:minimax-risk}
\end{equation}
where the infimum is over all measurable functions of the sample.

\paragraph{Blocks and their cost.}
A block $s:\R^{d_{\rm in}}\to\R^{d_{\rm out}}$ is a single-hidden-layer ReLU network with affine skip,
\begin{equation}
  s(x)=\sum_{k=1}^{K}v_k\relu(w_k^\top x-b_k)+Cx+c_0,
  \label{eq:rbv-block}
\end{equation}
where $\relu(u)=\max\{u,0\}$, $v_k\in\R^{d_{\rm out}}$, $w_k\in\R^{d_{\rm in}}$, $b_k\in\R$, $C\in\R^{d_{\rm out}\times d_{\rm in}}$, and $c_0\in\R^{d_{\rm out}}$. We call the terms $v_k\relu(w_k^\top x-b_k)$ the \emph{atoms} (neurons) of the block. With $e_r$ the $r$-th standard basis vector of $\R^{d_{\rm in}}$, the cost of the representation~\eqref{eq:rbv-block} is
\begin{multline}
  \norm{s}_{\RBV(d_{\rm in};d_{\rm out})}\defeq
  \sum_{k=1}^{K}\norm{v_k}_1\norm{w_k}_2\\
  +\sum_{j=1}^{d_{\rm out}}\Big(|s_j(0)|+\sum_{r=1}^{d_{\rm in}}|s_j(e_r)-s_j(0)|\Big).
  \label{eq:vector-rbv-norm}
\end{multline}
For finite-width networks, \citet[Lemma~2.14]{parhi2022kinds} express the norm of the vector-valued space $\RBV(\R^{d_{\rm in}};\R^{d_{\rm out}})$, with the $\ell^1$ sum over outputs, by this formula. We call the last two terms the \emph{anchor terms}, and~\eqref{eq:vector-rbv-norm} the \emph{anchored cost}; the anchor terms depend only on the function $s$. The first sum depends on the parameters $(v_k,w_k,b_k)$, for instance through splitting or merging atoms, so we treat~\eqref{eq:vector-rbv-norm} as the cost of a given representation and never use it as a norm of the function $s$.

\paragraph{Deep class.}
Fix integers $L\ge2$ and $w\ge1$. We compose $L$ blocks $s_\ell:\R^{d_{\ell-1}}\to\R^{d_\ell}$ with
\begin{equation}
  d_0=d_L=1,\qquad 1\le d_\ell\le w\quad(1\le\ell<L),
  \label{eq:intermediate-dims}
\end{equation}
each of the form~\eqref{eq:rbv-block} with at most $w$ atoms. The \emph{layer-sum representation cost} of $f$ is
\begin{equation}
  \VarCost_{L,w}(f)\defeq
  \inf_{f=s_L\circ\cdots\circ s_1}
  \sum_{\ell=1}^{L}\norm{s_\ell}_{\RBV(d_{\ell-1};d_\ell)},
  \label{eq:deep-cost}
\end{equation}
where the infimum ranges over all compositional representations satisfying these constraints. For $A,B>0$ let
\begin{multline}
  \cC_{L,w}(A,B)\defeq
  \big\{f:[0,2)\to\R:\\ \VarCost_{L,w}(f)\le A,\ \norm f_\infty\le B\big\},
  \label{eq:class-AB}
\end{multline}
and write $\Risk_n^*(A,B,\sigma)\defeq\Risk_n^*(\cC_{L,w}(A,B),\sigma)$. Every member of the class is the restriction to $[0,2)$ of a continuous piecewise-affine function on $\R$. It need not satisfy $f(0)=f(2-)$, so the circle enters only through the uniform design. The architecture has $O(Lw^2)$ parameters (Appendix~\ref{app:class}). The cost~\eqref{eq:deep-cost} is not positively homogeneous in $f$, so it is not a norm, although it is the quantity constrained in the deep variation spaces of \citet{parhi2022kinds}.

\paragraph{Notation.}
Throughout, $D=L-1$, $m=w-1$, and
\begin{equation}
  \kappa\defeq\Big(\frac AL\Big)^L,\quad
  H_L(A)\defeq\max_{1\le j\le L}\Big(\frac Aj\Big)^j,\quad
  q\defeq\frac{\sigma^2}{n}.
  \label{eq:KHq}
\end{equation}
By AM--GM, $\kappa$ bounds the product of $L$ block costs with sum $A$; it will control the oscillation of $f$. The quantity $H_L(A)\ge\kappa$ is the largest constant attainable within budget $A$ (Lemma~\ref{lem:constants}). In the lower-bound constructions, $\gamma$ denotes a total amplification gain. Constants $c,C>0$ are universal unless stated otherwise, and may change from line to line; named constants such as $c_M$, $C_A$, $C_A'$, and $D_0$ are fixed once chosen.

%% file: sections/03_results.tex
\section{Main results}
\label{sec:results}

The minimax risk is controlled from above in two independent ways: through the budget, which limits how much $f$ can oscillate, and through the architecture, which limits the number of parameters and the depth. We state these bounds first. Lower bounds then show that each is attained in its own regime: the architectural bound up to logarithmic factors when the budget is large (Section~\ref{sec:high}), and the budget bound up to constants when the budget is small (Section~\ref{sec:low}). Section~\ref{sec:critical} combines the two at one risk level. Between these regimes the rate is not determined.

\subsection{Upper bounds: oscillation and size}

For a continuous piecewise-affine $g:[0,2]\to\R$, let
\begin{equation}
  J(g)\defeq\TV\big(g';(0,2)\big)+|g'(0+)|+|g'(2-)|,
  \label{eq:J}
\end{equation}
the total variation of the derivative of the constant extension of $g$ to $\R$. This seminorm vanishes on constants and controls both the Lipschitz constant and the curvature of $g$. For $f\in\cC_{L,w}(A,B)$, $J(f)$ refers to the continuous extension of $f$ to $[0,2]$.

\begin{proposition}[Derivative variation]
\label{prop:variation}
Every $f\in\cC_{L,w}(A,B)$ satisfies $J(f)\le2\kappa=2(A/L)^L$. In particular $\Lip(f)\le \kappa$, and the distributional second derivative $f''$, a finite signed measure on $(0,2)$, has total variation at most $2\kappa$.
\end{proposition}

The bound has no factor depending on $w$, and it loses nothing across layers beyond the AM--GM step that produces $(A/L)^L$. It places $\cC_{L,w}(A,B)$, after centering, inside a classical ball of second-order total variation of radius $2\kappa$. The class is not equal to that ball: its members are realized by a fixed finite architecture. Combined with the exact cost of constant functions (Lemma~\ref{lem:constants}), the proposition gives the following upper bound.

\begin{theorem}[Upper bounds]
\label{thm:upper}
Let $\Dsig=\sigma^2+\min\{\kappa^2,B^2\}$. For all integers $L\ge2$, $w\ge1$, $n\ge2$ and all $A,B,\sigma>0$, the minimax risk $\Risk_n^*(A,B,\sigma)$ is at most $C$ times the minimum of
\begin{gather*}
  B^2,\qquad H_L(A)^2,\qquad
  q+\min\Big[\kappa^2,\ \kappa^{2/5}\Big(\frac{\Dsig}{n}\Big)^{4/5}\Big],\\
  \frac{(\sigma^2+B^2)L^2w^2\log(2Lw)\log(en)}{n}.
\end{gather*}
\end{theorem}

The third term depends on $A$ but not on $w$. It is the one-dimensional second-variation rate with radius $\kappa$, plus the parametric cost $q$ of estimating the mean. The last term depends on $(L,w)$ but not on $A$, and comes from the pseudodimension $O(L^2w^2\log(Lw))$ of the architecture. With $B=\sigma=1$, the third term is the smaller one, up to logarithmic factors, when $\kappa^{2/5}n^{1/5}\lesssim L^2w^2$, that is, when $\kappa\lesssim(L^2w^2)^{5/2}/\sqrt n$. Since $\kappa=(A/L)^L$, this threshold on $A$ is of order $L$ whenever $\log(Lwn)=O(L)$. Below it, the budget rather than the architecture limits the risk.

\subsection{High-complexity regime: quadratic depth at small budgets}
\label{sec:high}

For an integer $d\ge1$, the depth of the coding network, let
\begin{equation}
  \begin{gathered}
  M_d=\big\lfloor c_Mm^2d^2\log m\big\rfloor,\qquad
  \Pi_d=4w\big(w^2+w^{3/2}\big)^{d-1},\\
  \Lambda_d(A)=\frac{\kappa}{M_d\Pi_d}=\frac{(A/L)^L}{M_d\Pi_d},
  \end{gathered}
  \label{eq:Md}
\end{equation}
where $c_M>0$ is a small universal constant. $M_d\asymp d^2w^2\log w$ is the number of bits a depth-$d$, width-$m$ unit-coefficient network can encode, and $\Lambda_d(A)$ is the largest amplitude to which our construction can lift that code under budget $A$.

\begin{theorem}[Depth-allocated local packing]
\label{thm:packing}
There are universal constants $c_M,a_0,a_1,a_2,c_{\rm out}>0$ and an integer $D_0\ge2$ such that the following holds. Let $w\ge D_0+1$ and let $d$ be an integer with $D_0\le d\le D$. For all $A,B>0$ and every $\lambda$ with $0<\lambda\le\min\{c_{\rm out}B,\Lambda_d(A)\}$, there is a family $\{F_z:z\in\cZ\}\subset\cC_{L,w}(A,B)$ with $\log|\cZ|\ge a_0M_d$ and, for all $z\ne z'$,
\[
  \norm{F_z}_{L^2(\mu)}\le a_2\lambda,\qquad
  a_1\lambda\le\norm{F_z-F_{z'}}_{L^2(\mu)}\le a_2\lambda .
\]
\end{theorem}

The family has diameter and minimum separation both proportional to $\lambda$, which is what Fano's inequality requires. It exists at every positive budget, at amplitudes up to $\Lambda_d(A)$. Taking $d=D$ maximizes the code length $M_D\asymp L^2w^2\log w$. Smaller $d$ shortens the code by the factor $(d/D)^2$ but enlarges $\Pi_d^{-1}$, and hence $\Lambda_d(A)$, by the factor $(w^2+w^{3/2})^{D-d}$.

\begin{theorem}[Radius-explicit lower bound]
\label{thm:general-lower}
Under the assumptions of Theorem~\ref{thm:packing}, and $L\ge D_0+1$, for all $A,B,\sigma>0$ and $n\ge1$,
\[
  \Risk_n^*(A,B,\sigma)\ge
  c\max_{D_0\le d\le D}\ \min\Big\{B,\ \sigma\sqrt{\frac{M_d}{n}},\ \Lambda_d(A)\Big\}^2 .
\]
\end{theorem}

The three amplitudes are the output cap, the Gaussian testing scale, and the reach of the construction. The bound has a clean form when the output range and noise level are fixed.

\begin{corollary}[Sufficient budgets with $B=\sigma=1$]
\label{cor:budgets}
There are universal constants $c,C,C_A,C_A'>0$ and $w_0$ such that the following holds for $w\ge w_0$. If $D_0\le d\le D$ and $A\ge C_ALw^{(2d+2)/L}$, then $\Risk_n^*(A,1,1)\ge c\min\{1,M_d/n\}$ for all $n\ge1$. In particular:
\begin{enumerate}[leftmargin=*,itemsep=1pt,topsep=2pt,label=(\alph*)]
  \item if $L\ge C$, then
  \[
    \Risk_n^*(C_ALw^2,1,1)\ge c\min\Big\{1,\frac{L^2w^2\log w}{n}\Big\};
  \]
  \item if $0<\varepsilon\le2$ and $L\ge C/\varepsilon$, then
  \[
    \Risk_n^*(C_ALw^\varepsilon,1,1)\ge c\,\varepsilon^2\min\Big\{1,\frac{L^2w^2\log w}{n}\Big\};
  \]
  \item if $L\ge C\log w$, then
  \[
    \Risk_n^*(C_A'L,1,1)\ge c\min\Big\{1,\frac{L^2w^2}{n\log w}\Big\}.
  \]
\end{enumerate}
\end{corollary}

By Theorem~\ref{thm:upper}, all three bounds are at most $C\min\{1,L^2w^2\log(2Lw)\log(en)/n\}$, at every budget. Item~(a) uses the full code and needs a budget of order $Lw^2$. Items~(b) and~(c) shorten the coding network. In~(b) the budget constant is universal, while the risk constant is proportional to $\varepsilon^2$. So (b) holds for each fixed $\varepsilon$ but does not give a budget $O(L)$ with a constant independent of $w$. Item~(c) gives such a budget at the price of a factor $(\log w)^2$ in the rate. In the normalization $A=B=R$, $\sigma\asymp R$ of \citet{ganguli2026dichotomy}, Theorem~\ref{thm:general-lower} gives $\Risk_n^*\gtrsim R^2M_d/n$ whenever $n\ge M_d$ and $\Lambda_d(R)\ge R$.

The same construction also shows that the budget limits only the scale of the class, not its combinatorial dimension. Applied with $d=D$ to all $2^{M_D-1}$ binary strings, it fat-shatters $M_D-1$ points at margin proportional to $\min\{c_{\rm out}B,\Lambda_D(A)\}$. Hence, under the assumptions of Theorem~\ref{thm:packing} with $d=D$, $\Pdim(\cC_{L,w}(A,B))\ge M_D-1$ for every $A,B>0$ (Proposition~\ref{prop:fat}), within a factor $C\log(2Lw)/\log w$ of the upper bound. No argument based only on the pseudodimension can give a depth-linear rate at any budget.

\subsection{Low-budget regime: a sharp rate}
\label{sec:low}

When the budget is small, the oscillation bound of Proposition~\ref{prop:variation} is the binding constraint, and it is attained.

\begin{theorem}[Minimax rate at low budget]
\label{thm:low}
Let $w\ge3$, $n\ge2$, and $C_1>0$. Assume
\[
  \min\{\kappa,B\}\le C_1\sigma
  \quad\text{and}\quad
  \Big(\frac{\kappa^2}{q}\Big)^{1/5}\le\Big\lfloor\frac w3\Big\rfloor.
\]
Then $c\,\Phi\le\Risk_n^*(A,B,\sigma)\le C_{C_1}\Phi$, where
\[
  \Phi=\min\Big\{B^2,\ H_L(A)^2,\ q+\kappa^{2/5}q^{4/5}\Big\},
\]
$c$ is universal and $C_{C_1}$ depends only on $C_1$.
\end{theorem}

Each term of $\Phi$ has a clear source. $B^2$ is the output cap. $H_L(A)^2$ caps the size of constant functions, since a constant $a$ costs exactly $\min_jj|a|^{1/j}$. The term $q$ is the cost of estimating the mean, and $\kappa^{2/5}q^{4/5}$ is the second-variation rate with radius $\kappa$, attained by a spline hypercube built from $k_*\asymp(\kappa^2/q)^{1/5}$ disjoint tent functions. The width condition requires that these tents fit into one block, at three atoms each. The noise condition ensures that the centered functions, whose sup norm is at most $\min\{\kappa,2B\}$, are at most of the order of the noise level, as in the normalization $B\asymp\sigma$. If instead the noise is very small, $\sigma^2\le \kappa^2/n^3$ and $B\ge \kappa/n$, the risk is of order $\kappa^2/n^3$, set by the gaps of the random design (Proposition~\ref{prop:smallnoise}).

\begin{figure*}[t]
\centering
\input{figures/depth_allocation}
\caption{The representation of $\gamma N\circ r$ in Lemma~\ref{lem:allocation}, with the cost of each block written below it. Nodes are the coordinates after each block: the fold outputs $(r_1r(t),r_1)$, a wide block outputs at most $w$ coordinates, and the remaining blocks output scalars. Dotted segments stand for omitted wide blocks; blue marks the scalar gains, which cost $1$ per unit of gain instead of $w^2+w^{3/2}$. (a)~The code uses the whole depth, $d=D$. (b)~The code uses $d$ blocks after the fold, and the last $D-d$ blocks only amplify. In both panels the block count is $L$, the gains satisfy $\prod_\ell r_\ell=\gamma$, and balancing the costs gives total cost $L(\gamma\Pi_d)^{1/L}$, where $\Pi_d$ is the product of the per-unit costs.}
\label{fig:allocation}
\end{figure*}

\subsection{The critical budget at the linear-depth scale}
\label{sec:critical}

The packing amplitude $\Lambda_d(A)$, the sufficient budgets of Corollary~\ref{cor:budgets}, and the smallest budget at which the class actually reaches a given risk are different quantities. The first two describe a construction. The third is intrinsic, and requires an upper bound. Combining Corollary~\ref{cor:budgets}(c) with Theorem~\ref{thm:upper} determines it up to constants for the risk level in~(c), but not for the larger level in~(a) and~(b).

\begin{proposition}[Critical budget]
\label{prop:critical}
Let $C_A'$ be the constant of Corollary~\ref{cor:budgets}. There are universal constants $c_*,c',C,w_0>0$ such that, for $w\ge w_0$, $L\ge C\log w$, $M_*=L^2w^2/\log w$, and $M_*\le n\le M_*^5$,
\begin{gather*}
  \Risk_n^*(C_A'L,1,1)\ge c_*\frac{M_*}{n},\\
  \Risk_n^*(A,1,1)<c_*\frac{M_*}{n}\quad\text{for all } A\le c'L .
\end{gather*}
Hence the critical budget $A_{\rm crit}(c_*)\defeq\inf\{A:\Risk_n^*(A,1,1)\ge c_*M_*/n\}$ satisfies $c'L\le A_{\rm crit}(c_*)\le C_A'L$; that is, $A_{\rm crit}(c_*)=\Theta(L)$.
\end{proposition}

The proposition concerns the risk level $M_*/n=L^2w^2/(n\log w)$. For the larger level $cM_D/n$, with $M_D\asymp L^2w^2\log w$ and $c$ independent of $L$ and $w$, the same argument leaves a gap. Theorem~\ref{thm:upper} shows that reaching this level requires $\kappa\gtrsim M_D^{5/2}/\sqrt n$, that is, $A\ge L(c'M_D^{5/2}/\sqrt n)^{1/L}$, which is of order $L$ when $\log w=O(L)$ and $M_D\le n\le M_D^5$. Corollary~\ref{cor:budgets} gives sufficiency only at $A\asymp Lw^{c}$ for a constant $c>0$.

\begin{remark}[How the two regimes meet]
\label{rem:regimes}
The entropy bound behind Theorem~\ref{thm:upper} constrains every packing, not only estimators. If the class contains $e^M$ functions with pairwise $L^2(\mu)$ distances at least $\lambda$, then $M\le C(1+\sqrt{\kappa/\lambda})+\log(1+8B/\lambda)$ (Corollary~\ref{cor:packing-bound}). Once $M$ exceeds a large constant multiple of $1+\log(1+B/\lambda)$, this forces $\kappa\ge c\lambda M^2$. The packings of Theorem~\ref{thm:packing} use $\kappa\ge\lambda M_d\Pi_d$, and $\Pi_d\ge4w^{2d-1}\ge M_d$, so they pay at least this necessary amount. The ratio $\Pi_d/M_d$ is exactly what separates the sufficient budgets of Corollary~\ref{cor:budgets} from the necessary ones. At the testing scale $\lambda\asymp\sigma\sqrt{M/n}$, the necessary condition reads $\kappa\gtrsim\sigma M^{5/2}/\sqrt n$, as stated above.
\end{remark}

%% file: figures/depth_allocation.tex
% Depth-allocated amplification (Lemma 4.2): two aligned panels.
% Requires tikz libraries calc, decorations.pathreplacing and the colour accentblue.
\begin{tikzpicture}[
  x=1cm, y=1cm,
  nd/.style={circle, draw=black, line width=0.5pt, fill=white, inner sep=0pt, minimum size=5pt},
  gn/.style={circle, draw=accentblue, line width=0.6pt, fill=white, inner sep=0pt, minimum size=5pt},
  wire/.style={draw=black!45, line width=0.3pt},
  gwire/.style={draw=accentblue, line width=0.7pt},
  skip/.style={draw=black!45, line width=0.4pt, dotted},
  brc/.style={decorate, decoration={brace, amplitude=3pt}, line width=0.4pt},
  lab/.style={font=\footnotesize, inner sep=1pt},
  cost/.style={font=\footnotesize, inner sep=1pt, anchor=north, align=center},
]
\newcommand{\widecol}[2]{%
  \node[nd] (#2-1) at (#1, 0.75) {};
  \node[nd] (#2-2) at (#1, 0.40) {};
  \node[nd] (#2-3) at (#1, 0.05) {};
  \node[font=\scriptsize, inner sep=0pt] at (#1,-0.33) {$\vdots$};
  \node[nd] (#2-4) at (#1,-0.75) {};
}
\newcommand{\foldglyph}[1]{% graph of the fold r(t) = t - 2(t-1)_+ on [0,2]
  \draw[line width=0.3pt, black!50] ($(#1)+(-0.24,0)$) -- ($(#1)+(0.24,0)$);
  \draw[line width=0.55pt] ($(#1)+(-0.2,0)$) -- ($(#1)+(0,0.2)$) -- ($(#1)+(0.2,0)$);
}
\newcommand{\foldblock}[1]{% input t and the two outputs (r_1 r(t), r_1) of the fold block
  \node[nd] (#10) at (0,0) {};
  \node[lab, left=1pt] at (#10.west) {$t$};
  \node[nd] (#11t) at (1.5, 0.30) {};
  \node[nd] (#11b) at (1.5,-0.30) {};
  \draw[wire] (#10) -- (#11t); \draw[wire] (#10) -- (#11b);
  \foldglyph{0.75,0.42}
  \node[lab, left=1pt] at (0.53,0.50) {$r$};
}

% ===================== (a) full-depth coding, d = D =====================
\begin{scope}[yshift=4.1cm]
  \foldblock{a}
  \widecol{3.0}{a2}
  \foreach \i in {1,...,4}{\draw[wire] (a1t) -- (a2-\i); \draw[wire] (a1b) -- (a2-\i);}
  \foreach \i in {1,...,4}{\draw[skip] (a2-\i) -- ++(1.3,0);}
  \node[font=\small] at (5.25,0) {$\cdots$};
  \widecol{7.5}{a7}
  \foreach \i in {1,...,4}{\draw[skip] (a7-\i) -- ++(-1.3,0);}
  \node[nd] (a8) at (10.5,0) {};
  \foreach \i in {1,...,4}{\draw[wire] (a7-\i) -- (a8);}
  \node[lab, right=2pt] at (a8.east) {$\gamma N(r(t))$};
  \draw[brc] (1.45,1.05) -- node[lab, above=3pt] {coding network $N$, depth $D$} (10.55,1.05);
  \node[cost] at (0.75,-1.0) {$4r_1$};
  \node[cost] at (5.25,-1.0) {$(w^2{+}w^{3/2})\,r_\ell$\\[3pt] $2\le\ell\le D$};
  \node[cost] at (9.0,-1.0) {$w\,r_L$};
  \node[font=\small, anchor=west] at (-0.55,-2.25) {(a)\ \ $d=D$:\quad $\Pi_D=4w(w^2+w^{3/2})^{D-1}$};
\end{scope}

% ===================== (b) depth allocation, d < D =====================
\begin{scope}
  \foldblock{b}
  \widecol{3.0}{b2}
  \foreach \i in {1,...,4}{\draw[wire] (b1t) -- (b2-\i); \draw[wire] (b1b) -- (b2-\i);}
  \foreach \i in {1,...,4}{\draw[skip] (b2-\i) -- ++(0.5,0);}
  \node[font=\small] at (3.75,0) {$\cdots$};
  \widecol{4.5}{b4}
  \foreach \i in {1,...,4}{\draw[skip] (b4-\i) -- ++(-0.5,0);}
  \node[nd] (b5) at (6.0,0) {};
  \foreach \i in {1,...,4}{\draw[wire] (b4-\i) -- (b5);}
  \node[gn] (b6) at (7.5,0) {};
  \node[gn] (b7) at (8.6,0) {};
  \node[gn] (b8) at (10.5,0) {};
  \draw[gwire] (b5) -- (b6) -- (b7) -- ++(0.5,0);
  \node[font=\small, text=accentblue] at (9.55,0) {$\cdots$};
  \draw[gwire] (b8) -- ++(-0.5,0);
  \node[lab, right=2pt] at (b8.east) {$\gamma N(r(t))$};
  \draw[brc] (1.45,1.05) -- node[lab, above=3pt] {coding network $N$, depth $d$} (6.05,1.05);
  \draw[brc, accentblue] (6.4,1.05) -- node[lab, above=3pt, text=accentblue] {$D-d$ scalar gains $u\mapsto r_\ell u$} (10.55,1.05);
  \node[cost] at (0.75,-1.0) {$4r_1$};
  \node[cost] at (3.6,-1.0) {$(w^2{+}w^{3/2})\,r_\ell$\\[3pt] $2\le\ell\le d$};
  \node[cost] at (5.35,-1.0) {$w\,r_{d+1}$};
  \node[cost, text=accentblue] at (9.0,-1.0) {$r_\ell$\\[3pt] $d+2\le\ell\le L$};
  \node[font=\small, anchor=west] at (-0.55,-2.25) {(b)\ \ $d<D$:\quad $\Pi_d=4w(w^2+w^{3/2})^{d-1}$};
\end{scope}
\end{tikzpicture}

%% file: sections/04_proofs.tex
\section{Proof overview and key lemmas}
\label{sec:overview}

\subsection{Lower bounds: coding, amplification, testing}

Let $\cN(W,d,1)$ denote the scalar functions on $\R$ realized by $h_\ell=\relu(A_\ell h_{\ell-1}+b_\ell)$ for $1\le\ell<d$, followed by $N=A_dh_{d-1}+b_d$, with $h_0=x$, hidden width at most $W$, and all entries of all $A_\ell,b_\ell$ in $[-1,1]$. The coding capacity comes from bit extraction with bounded coefficients.

\begin{theorem}[Unit-coefficient approximation; \citealp{ou2024quantization}, Thm.~3.1]
\label{thm:corrected-approx}
There are universal $\Capp,D_0>0$ such that for all integers $m,d\ge D_0$ and every $g:[0,1]\to\R$ with $\norm g_\infty\le1$ and $\Lip(g)\le1$, some $N\in\cN(m,d,1)$ satisfies $\norm{N-g}_{L^\infty([0,1])}\le\Capp/(m^2d^2\log m)$.
\end{theorem}

The cited proof rescales coefficients by multiplying every pair $(A_\ell,b_\ell)$ by one factor \citep[Lemma~H.5]{ou2024quantization}. With biases this does not rescale the realized function homogeneously: $\relu(x)+1$ becomes $a^2\relu(x)+a$. Appendix~\ref{app:corrected-approx} replaces this step by an exact rescaling with one extra constant coordinate. The statement of the theorem is unaffected.

For $z\in\{0,1\}^{M_d-1}$, let $h_z$ interpolate the values $z_j$ at the grid points $j/M_d$, with $z_0=z_{M_d}=0$. Then $\norm{h_z-h_{z'}}_{L^2}^2\ge d_H(z,z')/(3M_d)$, and a Varshamov--Gilbert subset $\cZ$ with $\log|\cZ|\gtrsim M_d$ is separated at constant scale. Applying Theorem~\ref{thm:corrected-approx} to the bounded $1$-Lipschitz functions $h_z/M_d$ gives $N_z\in\cN(m,d,1)$ such that $Q_z=M_dN_z$ satisfies $\norm{Q_z-h_z}_\infty\le\eta<1/4$ for \emph{every} $z$, if $c_M$ is small (Lemma~\ref{lem:grid}). The code is separated at scale one, but the unit-coefficient networks $N_z$ have amplitude about $1/M_d$. Testing at amplitude $\lambda$ requires the gain $\gamma=\lambda M_d$, and this gain must be paid for within the budget.

\begin{lemma}[Depth-allocated amplification]
\label{lem:allocation}
Let $w\ge2$, $2\le d\le D=L-1$, $N\in\cN(w-1,d,1)$ of exact depth $d$, and $\gamma>0$. Let $r(t)=t-2\relu(t-1)$ on $[0,2]$. Then $t\mapsto \gamma N(r(t))$ has a representation with $L$ blocks, intermediate dimensions and atom counts at most $w$, and
\[
  \VarCost_{L,w}(\gamma N\circ r)\le L\,(\gamma\Pi_d)^{1/L}.
\]
\end{lemma}

\begin{proof}[Construction]
Choose gains $r_1,\dots,r_L>0$. The first block is $u_1(t)=r_1(r(t),1)$. It folds $[0,2]$ onto $[0,1]$ and creates a constant coordinate, at cost $4r_1$: the ridge contributes $2r_1$ and the anchor terms $2r_1$. For $2\le\ell\le d$, the block $u_\ell=\relu\big(r_\ell\big[\begin{smallmatrix}A_{\ell-1}&b_{\ell-1}\\0&1\end{smallmatrix}\big]u_{\ell-1}\big)$ has no bias, so by positive homogeneity
\[
  u_\ell=\Big(\prod_{i\le\ell}r_i\Big)\big(h_{\ell-1}(r(t)),\,1\big).
\]
Its cost is at most $(w^2+w^{3/2})r_\ell$: at most $w^{3/2}r_\ell$ from the ridges and $w^2r_\ell$ from the anchors. Block $d+1$ is the affine map $r_{d+1}[A_d\ b_d]$, of cost at most $wr_{d+1}$, which outputs $(\prod_{i\le d+1}r_i)N(r(t))$. The remaining $D-d$ blocks are scalar gains $u\mapsto r_\ell u$ of cost exactly $r_\ell$. The per-unit-gain costs $c_1=4$, $c_\ell=w^2+w^{3/2}$ for $2\le\ell\le d$, $c_{d+1}=w$, and $c_\ell=1$ for $\ell>d+1$ multiply to $\Pi_d$. Choosing $r_\ell=(\gamma\Pi_d)^{1/L}/c_\ell$ makes $\prod_\ell r_\ell=\gamma$ and the total cost $\sum_\ell c_\ell r_\ell=L(\gamma\Pi_d)^{1/L}$.
\end{proof}

The lemma combines three choices. Balanced gains spread the amplification over all $L$ blocks, and the fold is scaled together with the code, so there is no additive overhead. Depth is allocated so that wide layers, which pay $O(w^2)$ per unit gain, are used only where they encode bits. Figure~\ref{fig:allocation} compares the two extremes. Since $\Pi_d\le2^{d+1}w^{2d-1}\le2^Lw^{2d}$, the cost bound is at most $2L\gamma^{1/L}w^{2d/L}$. With full-depth coding, panel~(a), the width enters as $w^{2D/L}\approx w^2$. This is why a budget of order $Lw^2$ is needed when every layer carries the gain. Shortening the code to $d\approx\varepsilon L/2$ or $d\approx L/(2\log w)$ reduces this factor to at most $w^{\varepsilon}$ or $e$, which gives Corollary~\ref{cor:budgets}(b) and~(c). The code still has $M_d\asymp d^2w^2\log w$ bits, so it loses only the factor $(d/L)^2$ in the exponent of the packing. With $\gamma=\lambda M_d$, membership in the budget becomes $\lambda\le\Lambda_d(A)$, and the fold preserves $L^2$ distances. Together with $\norm{F_z}_\infty\le(1+\eta)\lambda$, this proves Theorem~\ref{thm:packing} for $F_z=\lambda Q_z\circ r$ (Appendix~\ref{app:packing}). For Theorem~\ref{thm:general-lower}, take $\lambda=c\min\{B,\sigma\sqrt{M_d/n},\Lambda_d(A)\}$. The sample laws $P_z$ share the design, so
\begin{align*}
  \KL(P_z\|P_{z'})&=\frac{n}{2\sigma^2}\norm{F_z-F_{z'}}_{L^2(\mu)}^2\\
  &\le\frac{a_2^2c^2}{2}M_d\le\frac1{16}\log|\cZ|
\end{align*}
for $c$ small. Fano's inequality then bounds the probability of decoding $z$ wrongly from below by a constant. Since the codewords are $a_1\lambda$-separated, nearest-neighbor decoding succeeds whenever $\norm{\widehat f-F_z}_2<a_1\lambda/2$, and Markov's inequality gives risk $\Omega(\lambda^2)$. For Corollary~\ref{cor:budgets}, one checks that $A\ge C_ALw^{(2d+2)/L}$ implies $\Lambda_d(A)\ge1$, using $M_d\Pi_d\le2^{d+1}L^2w^{2d+2}$.

\subsection{Upper bounds: an anchored comparison}

Proposition~\ref{prop:variation} rests on a one-block contraction. For a continuous piecewise-affine path $g:[0,2]\to\R^{d_{\rm in}}$, write $\mathbf J(g)=(J(g_1),\dots,J(g_{d_{\rm in}}))$.

\begin{lemma}[Anchored comparison]
\label{lem:anchored}
For every block $s$ of the form~\eqref{eq:rbv-block} and every continuous piecewise-affine path $g:[0,2]\to\R^{d_{\rm in}}$,
\[
  \norm{\mathbf J(s\circ g)}_2\le\big(\norm{s}_{\RBV}-\norm{s(0)}_1\big)\,\norm{\mathbf J(g)}_2 .
\]
\end{lemma}

The proof starts from an exact rewriting. With $\Delta_{jr}=s_j(e_r)-s_j(0)$, $\alpha_{kr}\in[0,1]$ defined by $\relu(w_{kr}-b_k)-\relu(-b_k)=\alpha_{kr}w_{kr}$, and $\odot$ the entrywise product,
\begin{gather*}
  s(z)=s(0)+\Delta z+\sum_kv_k\,\phi_k(z),\\
  \phi_k(z)=\max\{\ell_k^+(z),\,\ell_k^-(z)\}-\relu(-b_k),\\
  \ell_k^+(z)=\langle(1-\alpha_k)\odot w_k,z\rangle-b_k,\\
  \ell_k^-(z)=-\langle\alpha_k\odot w_k,z\rangle .
\end{gather*}
The variation $J$ of a maximum of two functions is at most the sum of their variations (Lemma~\ref{lem:lattice}; \citealp[Lemma~2]{huster2018limitations} treat the case $\max\{u,0\}$), and the two linear forms have complementary coefficients. The lattice inequality holds because, after extending by constants, $J$ is twice the total downward jump of the derivative, and a maximum creates no new downward jumps: where the two branches meet, the left slope of the maximum is the smaller left slope and its right slope is the larger right slope. Each ridge therefore satisfies $J(\phi_k\circ g)\le\sum_r|w_{kr}|J(g_r)\le\norm{w_k}_2\norm{\mathbf J(g)}_2$. Summing the output coordinates in $\ell^1$,
\begin{align*}
  \norm{\mathbf J(s\circ g)}_2&\le\sum_jJ(s_j\circ g)\\
  &\le\Big(\sum_{j,r}|\Delta_{jr}|+\sum_k\norm{v_k}_1\norm{w_k}_2\Big)\norm{\mathbf J(g)}_2 .
\end{align*}
The bracket is exactly the cost of $s$ minus $\norm{s(0)}_1$. No width factor appears because each norm meets its dual: outputs are summed in $\ell^1$, as in the cost, and inputs enter through $\langle w_k,\cdot\rangle$, which Cauchy--Schwarz bounds by $\norm{w_k}_2$. Iterating from the identity path, for which $J=2$, and applying AM--GM gives $J(f)\le2\kappa$. The same computation gives $\Lip(s)\le\norm{s}_{\RBV}-\norm{s(0)}_1$, which yields the exact cost of constants.

\begin{lemma}[Constants]
\label{lem:constants}
For every constant $a$, $\VarCost_{L,w}(a)=\min_{1\le j\le L}j|a|^{1/j}$. Every $f\in\cC_{L,w}(A,B)$ satisfies $\norm f_\infty\le2H_L(A)$ and $\norm f_{L^2(\mu)}^2\le\frac43H_L(A)^2$.
\end{lemma}

For the estimator behind Theorem~\ref{thm:upper}, write $f=\theta+g$ with $\theta=\int f\dd\mu$. The centered part satisfies $J(g)\le2\kappa$ and $\norm g_\infty\le\min\{\kappa,2B\}$. By the Jordan decomposition of the measure $g''$, $g$ is an affine function plus a difference of two convex functions of size $O(\kappa)$, so the centered class has $\eps$-covering numbers in the sup norm of logarithm $O(\sqrt{\kappa/\eps})$ \citep{guntuboyina2013covering}. We estimate $\theta$ by $\overline Y$ and fit $g$ by least squares over a finite net, using differences $Y_{2i}-Y_{2i-1}$ so that the unknown mean does not enter. Both estimates are centered, so their errors add without a cross term. A net of log-cardinality $k$ approximates $g$ within $O(\kappa/k^2)$ in sup norm, so least squares over it has risk of order $\kappa^2/k^4+\Dsig k/n$. Balancing the two terms at $k\asymp(\kappa^2n/\Dsig)^{1/5}$ gives $\kappa^{2/5}(\Dsig/n)^{4/5}$, which is the origin of the exponent $4/5$. If this $k$ is below one, the sample mean alone attains $q+O(\kappa^2)$. This gives the third term of the theorem, while the fourth is the pseudodimension bound (Appendix~\ref{app:upper}).

\subsection{Low-budget lower bounds}

Constants $\pm a$ with $|a|\le\min\{B,H_L(A)\}$ belong to the class, and a two-point test gives $\Risk_n^*\gtrsim\min\{B^2,H_L(A)^2,q\}$. For the nonparametric term, place $k\le\lfloor w/3\rfloor$ disjoint tents of height $\lambda$ in a single block, at cost $16\lambda k^2$. Scaling this block by $(A/L)^{-(L-1)}$ and following it with $L-1$ scalar gains $u\mapsto(A/L)u$ realizes every $0/1$ combination of tents within budget $A$ whenever $\lambda\le \kappa/(16k^2)$ (Figure~\ref{fig:spline} in Appendix~\ref{app:low}). Assouad's lemma then gives $\Risk_n^*\gtrsim\max_{k\le w/3}\min\{B^2,qk,\kappa^2/k^4\}$.

To prove Theorem~\ref{thm:low}, put $\tau=\kappa^{2/5}q^{4/5}$. The noise condition enters only the upper bound: it gives $\Dsig\le(1+C_1^2)\sigma^2$, so the third term of Theorem~\ref{thm:upper} is at most a constant times $q+\tau$. For the lower bound there are two cases. If $\kappa^2\le q$, then $\tau\le q$, and the constant functions already give $\Risk_n^*\gtrsim\min\{B^2,H_L(A)^2,q\}\ge\Phi/2$. If $\kappa^2>q$, then $k_*=(\kappa^2/q)^{1/5}>1$. The width condition gives $k=\lfloor k_*\rfloor\le\lfloor w/3\rfloor$, so the tents fit into one block, and $qk\ge\tau/2$ and $\kappa^2/k^4\ge\tau$. The spline bound is then of order $\min\{B^2,\tau\}$, which is at least $\Phi/2$ because $q<\tau$. In this case the cap $H_L(A)$ is inactive, since $\tau<\kappa^2\le H_L(A)^2$. The full argument is in Appendix~\ref{app:low}.

\subsection{The critical budget}

Proposition~\ref{prop:critical} combines the two directions at a single risk level. Sufficiency is Corollary~\ref{cor:budgets}(c): with $d=\lfloor L/(2\log w)\rfloor-1$, the code has $M_d\ge c\,M_*$ bits, so for $n\ge M_*$ the budget $C_A'L$ gives risk at least $c_*M_*/n$. For necessity, let $A\le c'L$ with $c'<1$. Then $\kappa=(A/L)^L\le c'^L$, and the third term of Theorem~\ref{thm:upper}, with $B=\sigma=1$ and $\Dsig\le2$, gives
\[
  \Risk_n^*(A,1,1)\le\frac Cn\big(1+2\kappa^{2/5}n^{1/5}\big)\le\frac Cn\big(1+2c'^{4/5}M_*\big),
\]
where the last step uses $n\le M_*^5$. The second-variation term exceeds the parametric rate $1/n$ by the factor $\kappa^{2/5}n^{1/5}$, which grows with $n$, so some upper limit on $n$ is unavoidable in this argument. Choosing $c'$ small in terms of $c_*$, and $w$ large so that $C\le c_*M_*/4$, makes the right-hand side smaller than $c_*M_*/n$. Both statements concern the same threshold $c_*M_*/n$, and the risk is nondecreasing in $A$, so $c'L\le A_{\rm crit}(c_*)\le C_A'L$.

%% file: sections/05_discussion.tex
\section{Discussion}
\label{sec:discussion}

Our results describe the minimax risk over $\cC_{L,w}(A,B)$ at the two ends of the budget range. At small budgets the class behaves like a one-dimensional smoothness class. The derivative-variation bound places it, after centering, in a second-variation ball of radius $2(A/L)^L$, and spline hypercubes show that this ball is attained. The risk therefore follows the $n^{-4/5}$ law, capped by the largest constant $H_L(A)$, as long as the optimal number of tents fits in one block. At budgets of order $Lw^\varepsilon$, and of order $L$ at the price of a factor $(\log w)^2$, the class already has local entropy quadratic in depth. For fixed $\varepsilon$, its risk is then within logarithmic factors of the architectural upper bound. For the risk level $L^2w^2/(n\log w)$, a budget linear in depth is also necessary.

Three budget notions should be kept apart. The amplitude $\Lambda_d(A)$ is what one representation reaches. The budget $C_ALw^{(2d+2)/L}$ suffices for one packing. The critical budget $A_{\rm crit}(c_*)$ is intrinsic to the class, and pinning it down needs an upper bound. For the level $L^2w^2\log w/n$ with uniform constants, Theorem~\ref{thm:upper} forces only $A\gtrsim L(M_D^{5/2}/\sqrt n)^{1/L}$. This is of order $L$ when $\log w=O(L)$ and $n$ is polynomial in $M_D$, while our sufficient budget is $Lw^{c}$. Remark~\ref{rem:regimes} locates the difficulty: a code with $M$ bits at amplitude $\lambda$ needs $\kappa\gtrsim\lambda M^2$, whereas ours uses $\kappa\asymp\lambda M_d\Pi_d$. Whether more efficient codes exist, or whether a sharper entropy bound makes the factor $w^c$ necessary, is open.

Between the two regimes, the minimax risk lies between the lower bounds of Theorem~\ref{thm:general-lower} and Lemma~\ref{lem:spline} and the upper bound of Theorem~\ref{thm:upper}. The spline lower bound saturates at $q\lfloor w/3\rfloor$, because one block holds at most $w/3$ tents. The architectural upper bound allows risk up to order $qL^2w^2\log(Lw)\log n$. Closing this gap would require either codes whose cost interpolates between a single wide block and a full bit-extraction network, or an entropy bound that uses depth and width jointly with the budget. Smaller gaps remain: a logarithmic factor in the high-complexity regime, the universal size thresholds inherited from Theorem~\ref{thm:corrected-approx}, and the computational cost of the least-squares estimators behind the upper bounds.

%% file: sections/A_architecture.tex
\section{Architecture conventions and computation-graph realization}
\label{app:class}

\subsection{The vector-valued class and the motivating formulation}
For a block $s(x)=\sum_{k=1}^{K}v_k\relu(w_k^\top x-b_k)+Cx+c_0$ we use the cost~\eqref{eq:vector-rbv-norm}, which is the formula of \citet[Lemma~2.14]{parhi2022kinds} for the space $\RBV(\R^{d_{\rm in}};\R^{d_{\rm out}})$ equipped with the $\ell^1$ sum over output coordinates. The deep space of \citet{parhi2022kinds} composes vector-valued maps across dimensions $d_0,\ldots,d_L$. \citet[Definition~1]{ganguli2026dichotomy} define their class through single-hidden-layer blocks of width at most $w$ and state an $O(Lw^2)$ parameter count, while their displayed block norm is the scalar one and the intermediate dimensions are not specified. Throughout this paper, all class inclusions refer to the explicit architecture
\[
  d_0=d_L=1,\qquad d_\ell\le w,\qquad K_\ell\le w .
\]
Taking $d_\ell=1$ for all $\ell$ gives the scalar chain, a subclass with $O(Lw)$ parameters that we do not study separately.

\subsection{Exact parameter count}
A block with input dimension $d_{\rm in}$, output dimension $d_{\rm out}$, and $K$ atoms has
\[
  d_{\rm out}K+Kd_{\rm in}+K+d_{\rm out}d_{\rm in}+d_{\rm out}
\]
scalar parameters. If $d_{\rm in},d_{\rm out},K\le w$, this is at most $3w^2+2w$. Hence
\begin{equation}
  W_{\rm par}\le L(3w^2+2w).
  \label{eq:parameter-count-full}
\end{equation}

\subsection{Computation graph and pseudodimension}
\begin{lemma}[Architecture to computation graph]
\label{lem:graph-main}
Every function in $\cC_{L,w}(A,B)$ is realized by a piecewise-linear computation graph with at most $CLw^2$ real parameters and computational depth at most $CL$. Consequently,
\begin{equation}
  \Pdim(\cC_{L,w}(A,B))\le CL^2w^2\log(2Lw).
  \label{eq:pdim-bound}
\end{equation}
\end{lemma}

\begin{proof}
For each block, compute the $K$ preactivations $w_k^\top x-b_k$ in one affine stage, apply $K$ ReLUs, and compute $\sum_kv_k\relu(\cdot)+Cx+c_0$ in a second affine stage. The skip $Cx$ is carried in parallel through the same block and adds no nonlinear depth. Thus an $L$-block composition is a piecewise-linear computation graph with at most the parameters in~\eqref{eq:parameter-count-full}, at most $Lw$ ReLU gates, and computational depth at most $2L$.

The skip is an identity-activation unit, which is piecewise polynomial with one piece, so the graph is covered by the piecewise-polynomial framework of \citet{bartlett2019nearly}; alternatively, writing $u=\relu(u)-\relu(-u)$ turns it into a plain ReLU network with a constant factor more parameters and the same depth. The finitely many choices of dimensions and atom counts are all subarchitectures of the maximal padded graph, obtained by setting unused parameters to zero.

The piecewise-linear pseudodimension bound of \citet{bartlett2019nearly}, applied to this maximal graph, gives
\[
  \Pdim\le CW_{\rm par}L\log(2W_{\rm par})\le CL^2w^2\log(2Lw).
\]
The layer-cost and output constraints only restrict the graph class and cannot increase the pseudodimension.
\end{proof}

\subsection{Circle coordinate and representation cost}
The coordinate $t\in[0,2)$ is the coordinate $t=\theta/\pi$ of \citet{ganguli2026dichotomy}, with normalized measure $\dd t/2$. The fold $r$ of Lemma~\ref{lem:allocation} satisfies $r(0)=r(2)=0$, so the packed functions periodize continuously. General members of $\cC_{L,w}(A,B)$ need not satisfy $f(0)=f(2-)$; the upper-bound proofs therefore work on the interval $[0,2]$ and never use periodicity.

The infimum~\eqref{eq:deep-cost} is generally not one-homogeneous as a functional of the composed map: a scalar gain $\gamma$ can be distributed across $L$ blocks at cost $L\gamma^{1/L}$. All arguments therefore use only the definition of the representation cost and never its homogeneity.

%% file: sections/B_approximation.tex
\section{The unit-coefficient approximation theorem with biases}
\label{app:corrected-approx}

This appendix proves Theorem~\ref{thm:corrected-approx}. The theorem is \citet[Theorem~3.1]{ou2024quantization}. The bit decoder, the realization of bounded piecewise-linear functions, and all approximation estimates are imported from their constructive proof. We re-derive only the coefficient-rescaling steps. Throughout this appendix, $D$ and $B$ are a generic depth and coefficient bound, unrelated to $D=L-1$ and the output bound, and $J$ and $r$ denote an additional depth and a fan-out factor, unrelated to the derivative variation and the fold, and $\cN(W,D,B)$ denotes the scalar functions realized by $h_\ell=\relu(A_\ell h_{\ell-1}+b_\ell)$ for $1\le\ell<D$ and $N=A_Dh_{D-1}+b_D$, with $h_0=x$, hidden width at most $W$, and every entry of every $A_\ell,b_\ell$ bounded by $B$. Depth counts affine maps. A network of smaller depth is padded to exact depth $D$ by identity layers before the output; since hidden states are nonnegative, $\relu(Ih)=h$.

\begin{lemma}[Homogeneous lift]
\label{lem:balanced}
Let $B\ge1$ and let $N\in\cN(W,D,B)$ have exact depth $D\ge2$. For every $\gamma>0$, $\gamma N\in\cN(W+1,D,B \gamma^{1/D})$.
\end{lemma}

\begin{proof}
Put $s=\gamma^{1/D}$ and define $\widetilde h_1=\relu\big(\big[\begin{smallmatrix}sA_1\\0\end{smallmatrix}\big]x+\big[\begin{smallmatrix}sb_1\\s\end{smallmatrix}\big]\big)=(sh_1,s)$ and, for $2\le\ell<D$, $\widetilde h_\ell=\relu\big(\big[\begin{smallmatrix}sA_\ell&sb_\ell\\0&s\end{smallmatrix}\big]\widetilde h_{\ell-1}\big)$. By induction $\widetilde h_\ell=(s^\ell h_\ell,s^\ell)$, and the output map $[\,sA_D\ \ sb_D\,]\widetilde h_{D-1}=s^DN=\gamma N$. The width grows by one, every old coefficient is multiplied by $s$, and the only new coefficient is $s\le sB$.
\end{proof}

\subsection{The rescaling issue}
\citet[Lemma~H.5]{ou2024quantization} asserts $(B'/B)^L\mathcal R(W,L,B)=\mathcal R(W,L,B')$ by multiplying every affine pair $(A_\ell,b_\ell)$ by $B'/B$ and invoking positive homogeneity of the ReLU. With biases, uniform multiplication scales homogeneous terms and biases by different powers across depth. For example, the depth-two realization
\[
  N(x)=\relu(x)+1
\]
becomes $a^2\relu(x)+a$, not $a^2N(x)$: the last-layer bias receives only one factor of $a$. This identity enters the proof of their Theorem~3.1 through their Proposition~H.4, which is used twice: in~(99), to pass from Proposition~B.1 to Theorem~3.1, and in~(211), inside the proof of the bounded-spline realization (their Proposition~B.3). The exact rescaling of Lemma~\ref{lem:balanced} repairs both uses. The unit-coefficient fan-out step of \citet[Lemma~H.6]{ou2024quantization} uses zero biases in all added layers and needs no change; we restate it for completeness.

\subsection{Trading coefficient magnitude for depth}

\begin{lemma}[Unit-coefficient fan-out]
\label{lem:fanout}
Let $W\ge2$, let $N\in\cN(W,D,1)$, and let $J\ge0$. Put $r=\lfloor W/2\rfloor$. Then
\begin{equation}
  r^JN\in\cN(W,D+J,1).
  \label{eq:fanout-inclusion}
\end{equation}
\end{lemma}

\begin{proof}
The case $J=0$ is immediate. For $J\ge1$, pad the realization to exact depth $D$ and write its scalar output as $y=A_Dh_{D-1}+b_D$. Replace this final affine map by the hidden vector
\[
  u_1=(\relu(y)\mathbf 1_r,\relu(-y)\mathbf 1_r)\in\R^{2r}.
\]
For each of the next $J-1$ hidden layers, apply $\operatorname{diag}(\mathbf 1_{r\times r},\mathbf 1_{r\times r})$ with zero bias. The final row $(\mathbf 1_r^\top,-\mathbf 1_r^\top)$ returns $r^J(\relu(y)-\relu(-y))=r^Jy$. All coefficients belong to $\{-1,0,1\}$, and $2r\le W$.
\end{proof}

\begin{corollary}[Depth--coefficient conversion]
\label{cor:depth-weight-corrected}
Let $W\ge4$, $D\ge2$, and $B\ge1$. Put
\begin{equation}
  r=\Big\lfloor\frac{W+1}{2}\Big\rfloor,\qquad
  J=\Big\lceil\frac{D\log B}{\log r}\Big\rceil,
  \label{eq:conversion-depth}
\end{equation}
with $J=0$ when $B=1$. Then
\begin{equation}
  \cN(W,D,B)\subseteq\cN(W+1,D+J,1).
  \label{eq:depth-weight-corrected}
\end{equation}
If $B=W^{\beta}$ for a fixed $\beta>0$, then $J\le2\beta D+1$.
\end{corollary}

\begin{proof}
For $N\in\cN(W,D,B)$, pad to exact depth $D$. Lemma~\ref{lem:balanced} with $\gamma=B^{-D}$ gives $B^{-D}N\in\cN(W+1,D,1)$. Apply Lemma~\ref{lem:fanout} at width $W+1$ to obtain $r^JB^{-D}N$. Since $r^J\ge B^D$, multiplying the final affine map by $B^D/r^J\le1$ recovers $N$ without violating the unit coefficient bound. Finally, $r^2\ge W$ for $W\ge4$, so $\log r\ge\frac12\log W$, and the stated bound on $J$ follows.
\end{proof}

\subsection{Repairing the bounded-spline realization}
For a strictly increasing breakpoint sequence $X=(x_i)_{i=0}^{M_X-1}\subset[0,1]$, let $\Sigma(X,E)$ denote the continuous functions that are constant outside $[x_0,x_{M_X-1}]$, affine between consecutive breakpoints, and bounded in absolute value by $E$. Put
\[
  R_m(X)=\max_{1\le i\le M_X-1}(x_i-x_{i-1})^{-1}.
\]
The direct construction of \citet[Proposition~C.1]{ou2024quantization}, which does not use the rescaling identity, implies that whenever $u^2v\ge M_X$,
\begin{equation}
  \Sigma\Big(X,\frac{1}{C_kM_X^6R_m(X)^4}\Big)\subseteq\cN(20u,30v,1),
  \label{eq:small-spline-import}
\end{equation}
where $2\le C_k\le10^5$ is universal.

\begin{lemma}[Bounded-spline realization]
\label{lem:corrected-spline}
Assume $u^2v\ge M_X$, $v\ge1$, $w_s\ge1$, and
\begin{equation}
  w_s^{30v}\ge M_X^6R_m(X)^4E.
  \label{eq:spline-condition}
\end{equation}
Then
\begin{equation}
  \Sigma(X,E)\subseteq\cN(20u+1,30v,2w_s).
  \label{eq:corrected-spline}
\end{equation}
\end{lemma}

\begin{proof}
For $f\in\Sigma(X,E)$, set $f_0=(2w_s)^{-30v}f$. By~\eqref{eq:spline-condition} and $2^{30v}\ge2^{30}>10^5\ge C_k$,
\[
  \norm{f_0}_\infty\le\frac{1}{C_kM_X^6R_m(X)^4}.
\]
Thus $f_0$ belongs to the left side of~\eqref{eq:small-spline-import}. Lemma~\ref{lem:balanced}, with $\gamma=(2w_s)^{30v}$ and coefficient bound $1$, realizes $(2w_s)^{30v}f_0=f$ at the same depth, with width at most $20u+1$ and coefficient magnitude at most $2w_s$.
\end{proof}

\subsection{Width and depth bookkeeping}
The remaining decoder identities and approximation estimates in the proof of \citet[Proposition~B.1]{ou2024quantization} are explicit and do not use the rescaling identity. Its bounded-spline proposition is called exactly twice, at equations~(144) and~(161) of the cited proof. Replacing those calls by Lemma~\ref{lem:corrected-spline} changes the two spline widths from $40m$ and $40n$ to $40m+1$ and $40n+1$, while preserving the realized functions, depths, and coefficient bounds.

The resulting additive width changes fit inside the slack of the original construction. In its first stage, the corrected parallelization and affine-combination bounds give
\[
  \operatorname{width}(f_1)\le200m+2^{n+5}+5.
\]
In the second stage,
\[
  \operatorname{width}(u)\le\max\{40m+1,40n+1\}+2.
\]
For $m,n\ge2$, the first display dominates the second, so the composition of $f_1$ with $u$ retains the first width bound. Three parallel copies followed by the fixed median network have width at most
\[
  600m+3\cdot2^{n+5}+15\le600m+2^{n+7}.
\]
Consequently the terminal architecture estimate of the cited bit-extraction proof remains valid without changing its asymptotic or stated width bound.

\begin{proposition}[Polynomial-coefficient approximation]
\label{prop:polyweight-corrected}
There are universal constants $C,D_1>0$ such that, for all integers $W,U\ge D_1$ and every continuous $g:[0,1]\to\R$ with $\norm g_\infty\le1$ and $\Lip(g)\le1$, there is $N\in\cN(W,U,W^2)$ satisfying
\begin{equation}
  \norm{N-g}_\infty\le\frac{C}{W^2U^2\log W}.
  \label{eq:polyweight-approx}
\end{equation}
\end{proposition}

\begin{proof}
The corrected construction just described yields, for integers $m,n,\ell\ge2$,
\begin{align}
  \operatorname{width}(N)&\le600m+2^{n+7},\notag\\
  \operatorname{depth}(N)&\le101\ell,\notag\\
  \operatorname{coeff}(N)&\le\max\{8mn,3^{n+2}\},\notag\\
  \norm{N-g}_\infty&\le\frac{3}{m^2\ell^2n}.
  \label{eq:imported-bookkeeping}
\end{align}
Choose
\[
  m=\Big\lfloor\frac{W}{1000}\Big\rfloor,\qquad
  \ell=\Big\lfloor\frac{U}{101}\Big\rfloor,
\]
and let $n\ge2$ be the largest integer such that $2^{n+7}\le W/5$. For sufficiently large $W,U$, one has $m\asymp W$, $\ell\asymp U$, and $n\asymp\log W$. Moreover,
\[
  600m+2^{n+7}\le\tfrac45W,\qquad 101\ell\le U,
\]
and $8mn\le W^2$, $3^{n+2}\le W^2$ for all sufficiently large $W$. Padding unused width and depth places the network in $\cN(W,U,W^2)$. Substituting the parameter choices into~\eqref{eq:imported-bookkeeping} proves~\eqref{eq:polyweight-approx}.
\end{proof}

\subsection{Proof of Theorem~\ref{thm:corrected-approx}}
Let $m,d$ be sufficiently large, set $\overline m=m-1$, and put
\[
  U=\Big\lfloor\frac{d-1}{5}\Big\rfloor.
\]
Proposition~\ref{prop:polyweight-corrected} gives $N_0\in\cN(\overline m,U,\overline m^2)$ with error at most $C/(\overline m^2U^2\log\overline m)$. Apply Corollary~\ref{cor:depth-weight-corrected}. With $r=\lfloor m/2\rfloor$ and $r^2\ge\overline m$ for sufficiently large $m$, the additional depth obeys
\[
  J=\Big\lceil\frac{U\log(\overline m^2)}{\log r}\Big\rceil\le4U+1.
\]
Hence the same function has a unit-coefficient realization of width at most $m$ and depth at most $U+J\le5U+1\le d$. Since $\overline m\asymp m$ and $U\asymp d$, the error is at most $\Capp/(m^2d^2\log m)$, which proves Theorem~\ref{thm:corrected-approx}.

%% file: sections/C_grid_code.tex
\section{The grid code and its local geometry}
\label{app:grid}

Fix an integer $d$ with $D_0\le d$ and $m=w-1\ge D_0$, and let $M=M_d$ be as in~\eqref{eq:Md}.

\begin{lemma}[Grid code]
\label{lem:grid}
There are universal constants $c_{\rm VG},c_1,c_2>0$ and $\eta\in(0,1/4)$ such that, if $c_M$ is sufficiently small and $M_d\ge8$, the following holds. For \emph{every} $z\in\{0,1\}^{M_d-1}$ there is $N_z\in\cN(m,d,1)$ such that $Q_z=M_dN_z$ satisfies $\norm{Q_z-h_z}_\infty\le\eta$, and hence $\norm{Q_z}_\infty\le1+\eta$. Moreover, there is $\cZ\subset\{0,1\}^{M_d-1}$ with $\log|\cZ|\ge c_{\rm VG}M_d$ such that $c_1\le\norm{Q_z-Q_{z'}}_{L^2([0,1])}\le c_2$ for all $z\ne z'$ in $\cZ$.
\end{lemma}

\begin{proof}
Write $M=M_d$. For $z\in\{0,1\}^{M-1}$, set $z_0=z_M=0$ and let $h_z$ linearly interpolate the points $(j/M,z_j)$, $j=0,\ldots,M$. Then $\norm{h_z}_\infty\le1$ and $\Lip(h_z)\le M$. On $I_j=[j/M,(j+1)/M]$, writing $u=Mx-j$ and $\delta_j=z_j-z_j'$, we have $h_z(x)-h_{z'}(x)=(1-u)\delta_j+u\delta_{j+1}$. Therefore
\begin{equation}
  \int_{I_j}(h_z-h_{z'})^2\dd x
  =\frac1M\int_0^1\big((1-u)\delta_j+u\delta_{j+1}\big)^2\dd u
  =\frac{\delta_j^2+\delta_j\delta_{j+1}+\delta_{j+1}^2}{3M}.
  \label{eq:grid-cell-full}
\end{equation}
Since $a^2+ab+b^2=\tfrac12(a^2+b^2)+\tfrac12(a+b)^2$, summing over $j$ and using $\delta_0=\delta_M=0$ yields
\begin{equation}
  \norm{h_z-h_{z'}}_{L^2([0,1])}^2\ge\frac{d_H(z,z')}{3M}.
  \label{eq:grid-lower-full}
\end{equation}
The Varshamov--Gilbert bound \citep{gilbert1952comparison,varshamov1957estimate}, in the form of \citet[Lemma~2.9]{tsybakov2009introduction}, gives $\cZ\subseteq\{0,1\}^{M-1}$ and universal $c_{\rm VG},c_H>0$ with $\log|\cZ|\ge c_{\rm VG}M$ and $d_H(z,z')\ge c_HM$ for $z\ne z'$. Since Hamming distances are at most $M-1<M$, $c_H<1$.

The functions $g_z=h_z/M$ satisfy $\norm{g_z}_\infty\le1$ and $\Lip(g_z)\le1$ for every $z\in\{0,1\}^{M-1}$. For each such $z$, Theorem~\ref{thm:corrected-approx} (with depth $d$) gives $N_z\in\cN(m,d,1)$ with $\norm{N_z-g_z}_\infty\le\Capp/(m^2d^2\log m)$. With $M=\lfloor c_Mm^2d^2\log m\rfloor$ and $c_M\le\eta/(2\Capp)$,
\[
  \norm{Q_z-h_z}_\infty=M\norm{N_z-g_z}_\infty\le\eta\qquad\text{for every }z\in\{0,1\}^{M-1}.
\]
This bound does not use the subset $\cZ$. Fix $\eta<\sqrt{c_H/3}/4<1/(4\sqrt3)$. For $z\ne z'$ in $\cZ$,
\[
  \norm{Q_z-Q_{z'}}_{L^2([0,1])}\ge\sqrt{c_H/3}-2\eta,\qquad
  \norm{Q_z-Q_{z'}}_{L^2([0,1])}\le1+2\eta .
\]
This proves the lemma with $c_1=\sqrt{c_H/3}-2\eta>0$ and $c_2=1+2\eta$.
\end{proof}

If an approximant has depth smaller than $d$, it is padded to exact depth $d$ by identity layers before its output map; if it is affine, $ax+b$ is first written as $a\relu(x)-a\relu(-x)+b$. Neither step changes the realized function, the width bound $m\ge2$, or the unit coefficient bound.

%% file: sections/D_packing.tex
\section{Depth-allocated amplification and the local packing}
\label{app:packing}

\subsection{Proof of Lemma~\ref{lem:allocation}}
Let $N\in\cN(m,d,1)$ with $m=w-1$ have exact depth $d\ge2$, with hidden maps $h_\ell=\relu(A_\ell h_{\ell-1}+b_\ell)$ for $1\le\ell<d$, $h_0=x$, and output $N=A_dh_{d-1}+b_d$. Let $r(t)=t-2\relu(t-1)$ on $[0,2]$, so that $r(t)=t$ on $[0,1]$, $r(t)=2-t$ on $[1,2]$, and $r(0)=r(2)=0$. Fix gains $r_1,\ldots,r_L>0$ and define the $L=D+1$ blocks as follows.

\emph{Block $1$ (fold).} $u_1(t)=r_1\,(r(t),1)\in\R^2$. In the form~\eqref{eq:rbv-block}, this block has the single atom $(-2r_1,0)\relu(t-1)$, skip coefficient $(r_1,0)$, and constant $(0,r_1)$. Its cost is $2r_1$ (ridge) $+\,r_1$ (first coordinate: $|r_1r(1)-r_1r(0)|=r_1$) $+\,r_1$ (second coordinate: $|u_{1,2}(0)|=r_1$), that is, $4r_1$.

\emph{Blocks $2,\ldots,d$ (lifted hidden layers).} For $2\le\ell\le d$ let $G_\ell=\big[\begin{smallmatrix}A_{\ell-1}&b_{\ell-1}\\0&1\end{smallmatrix}\big]$ and $u_\ell=\relu(r_\ell G_\ell u_{\ell-1})$. The matrix $G_\ell$ has at most $m+1=w$ rows and at most $w$ columns, and all its entries lie in $[-1,1]$. In the form~\eqref{eq:rbv-block}, $u_\ell=\sum_ie_i\relu(r_\ell g_i^\top u)$, where $e_i$ are the standard basis vectors and $g_i$ the rows of $G_\ell$, so the block has at most $w$ atoms and no bias. Its ridge cost is $r_\ell\sum_i\norm{g_i}_2\le w^{3/2}r_\ell$. Its anchor terms are $|u_{\ell,i}(0)|=0$ and $\sum_{i,c}|\relu(r_\ell G_{ic})|\le w^2r_\ell$. Hence its cost is at most $(w^2+w^{3/2})r_\ell$. By positive homogeneity of $\relu$ and induction,
\[
  u_\ell=\Big(\prod_{i\le\ell}r_i\Big)\big(h_{\ell-1}(r(t)),\,1\big),\qquad 2\le\ell\le d,
\]
since $\relu\big(r_\ell\prod_{i<\ell}r_i\,(A_{\ell-1}h_{\ell-2}+b_{\ell-1},1)\big)=\prod_{i\le\ell}r_i\,(h_{\ell-1},1)$, with $h_0(r(t))=r(t)$ in the case $\ell=2$.

\emph{Block $d+1$ (coding output).} $u_{d+1}=r_{d+1}[A_d\ \ b_d]\,u_d=\big(\prod_{i\le d+1}r_i\big)N(r(t))$. This affine block has no atoms, $u_{d+1}(0)=0$, input dimension $m+1=w$, and coefficients at most $r_{d+1}$, so its cost is at most $wr_{d+1}$.

\emph{Blocks $d+2,\ldots,L$ (scalar gains).} $u_\ell=r_\ell u_{\ell-1}$. Each is a scalar affine block with value $0$ at $0$ and anchor difference $r_\ell$, so its cost is exactly $r_\ell$.

The composition uses $1+(d-1)+1+(D-d)=L$ blocks. The intermediate dimensions are $2$, then $w$ (blocks $2,\ldots,d$), then $1$; all atom counts are at most $w$; $d_0=d_L=1$. The output is $(\prod_\ell r_\ell)N\circ r$. The costs are bounded by $c_\ell r_\ell$ with $c_1=4$, $c_2=\cdots=c_d=w^2+w^{3/2}$, $c_{d+1}=w$, and $c_\ell=1$ for $\ell>d+1$, so $\prod_\ell c_\ell=\Pi_d$. Set $r_\ell=(\gamma\Pi_d)^{1/L}/c_\ell$. Then $\prod_\ell r_\ell=\gamma\Pi_d/\Pi_d=\gamma$ and $\sum_\ell c_\ell r_\ell=L(\gamma\Pi_d)^{1/L}$, which proves the lemma.

When $d=D$ there are no scalar gains and $\Pi_D=4w(w^2+w^{3/2})^{L-2}$. In this case the construction is balanced amplification of the full code, with the fold scaled together with the code. The resulting scale $\Lambda_D(A)$ is at least the scale $\frac1M\big(\frac{\pos{A-3}}{3Dw^2}\big)^D$ obtained when the fold is charged separately and each of the $D$ coding layers carries the gain $(\lambda M)^{1/D}$. Indeed, using $w^2+w^{3/2}\le2w^2$, $A^D\ge\pos{A-3}^D$, and $(D/L)^D\ge e^{-1}$, the ratio of the two is at least $\frac{A}{L}\,e^{-1}(3/2)^D\,\frac w2$. This is at least one whenever $A>3$, $L\ge4$, and $w\ge3$.

\subsection{Proof of Theorem~\ref{thm:packing}}
Let $D_0$ be at least the threshold of Theorem~\ref{thm:corrected-approx}, large enough that $M_d\ge8$ whenever $d,m\ge D_0$. Take $\cZ$, $N_z$, and $Q_z$ from Lemma~\ref{lem:grid}, and set $F_z=\lambda Q_z\circ r=\lambda M_dN_z\circ r$.

\emph{Geometry.} For any measurable $u:[0,1]\to\R$,
\[
  \int_0^2|u(r(t))|^2\frac{\dd t}{2}=\frac12\int_0^1|u(t)|^2\dd t+\frac12\int_1^2|u(2-t)|^2\dd t=\int_0^1|u(x)|^2\dd x .
\]
Hence $\norm{F_z-F_{z'}}_{L^2(\mu)}=\lambda\norm{Q_z-Q_{z'}}_{L^2([0,1])}$ and $\norm{F_z}_{L^2(\mu)}\le\lambda(1+\eta)$, which gives the stated geometry with $a_0=c_{\rm VG}$, $a_1=c_1$, $a_2=\max\{c_2,1+\eta\}$.

\emph{Membership.} By Lemma~\ref{lem:allocation} with $\gamma=\lambda M_d$, $\VarCost_{L,w}(F_z)\le L(\lambda M_d\Pi_d)^{1/L}$. This is at most $A$ if and only if $\lambda\le(A/L)^L/(M_d\Pi_d)=\Lambda_d(A)$. Also $\norm{F_z}_\infty\le(1+\eta)\lambda\le B$ when $\lambda\le c_{\rm out}B$ with $c_{\rm out}=1/(1+\eta)$. The condition $r(0)=r(2)=0$ makes $F_z$ continuous on the circle. These bounds hold for every $z\in\{0,1\}^{M_d-1}$.

\subsection{Fat-shattering and pseudodimension}
Points $t_1,\ldots,t_k$ are \emph{$\delta$-shattered} by $\cF$ if there are thresholds $s_1,\ldots,s_k$ such that for every $b\in\{0,1\}^k$ some $f\in\cF$ has $f(t_j)\ge s_j+\delta$ when $b_j=1$ and $f(t_j)\le s_j-\delta$ when $b_j=0$. The fat-shattering dimension $\operatorname{fat}_\delta(\cF)$ is the largest such $k$.

\begin{proposition}[Fat-shattering and pseudodimension]
\label{prop:fat}
Under the assumptions of Theorem~\ref{thm:packing} with $d=D$, for all $A,B>0$ and $0<\lambda\le\min\{c_{\rm out}B,\Lambda_D(A)\}$, the points $j/M_D$, $1\le j\le M_D-1$, are $(\frac12-\eta)\lambda$-shattered by $\cC_{L,w}(A,B)$ with all thresholds equal to $\lambda/2$. Consequently, for all $A,B>0$,
\[
  M_D-1\le\Pdim\big(\cC_{L,w}(A,B)\big)\le CL^2w^2\log(2Lw).
\]
\end{proposition}

\begin{proof}
Write $M=M_D$. By the membership argument above, $F_z\in\cC_{L,w}(A,B)$ for \emph{every} $z\in\{0,1\}^{M-1}$, not only for $z\in\cZ$. The subset $\cZ$ is chosen for $L^2$ separation and does not realize every labeling, so it could not be used here. Let $t_j=j/M$. Since $t_j\in[0,1]$, $r(t_j)=t_j$ and $h_z(t_j)=z_j$, so $|F_z(t_j)-\lambda z_j|\le\eta\lambda$. Given $b\in\{0,1\}^{M-1}$, take $z=b$. If $b_j=1$, then $F_z(t_j)\ge(1-\eta)\lambda=\lambda/2+(\frac12-\eta)\lambda$. If $b_j=0$, then $F_z(t_j)\le\eta\lambda=\lambda/2-(\frac12-\eta)\lambda$. Since $\eta<1/4$, the margin is positive. For every $A,B>0$, $\min\{c_{\rm out}B,\Lambda_D(A)\}>0$, so an admissible $\lambda$ exists. Strict inequalities with witnesses $s_j=\lambda/2$ are pseudo-shattering, so $\Pdim\ge M-1$. The upper bound is Lemma~\ref{lem:graph-main}.
\end{proof}

%% file: sections/E_lower_bounds.tex
\section{Proofs of the lower bounds in the high-complexity regime}
\label{app:fano}

\subsection{Proof of Theorem~\ref{thm:general-lower}}
Fix $d$ with $D_0\le d\le D$; the theorem follows by taking the best $d$. Let $\cZ$ and $F_z$ be as in Appendix~\ref{app:packing}, so that
\begin{equation}
  \log|\cZ|\ge a_0M_d,\qquad
  a_1\lambda\le\norm{F_z-F_{z'}}_2\le a_2\lambda .
  \label{eq:fano-geometry-full}
\end{equation}
Choose $\lambda=c_\lambda\min\{B,\sigma\sqrt{M_d/n},\Lambda_d(A)\}$ with $c_\lambda\le c_{\rm out}$, so that the family lies in $\cC_{L,w}(A,B)$. Under the $n$-sample law $P_z^{(n)}$, the design law is common and the conditional responses are independent Gaussians with variance $\sigma^2$. Hence
\[
  \KL\big(P_z^{(n)}\|P_{z'}^{(n)}\big)
  =n\,\E_T\KL\big(N(F_z(T),\sigma^2)\|N(F_{z'}(T),\sigma^2)\big)
  =\frac{n}{2\sigma^2}\norm{F_z-F_{z'}}_2^2
  \le\frac{a_2^2c_\lambda^2}{2}M_d .
\]
Choose $c_\lambda$ so that the last term is at most $\frac1{16}\log|\cZ|$, and $D_0$ so that $\log|\cZ|\ge4\log2$. With the uniform prior on $\cZ$, the mutual information is at most the average pairwise divergence, and Fano's inequality \citep{tsybakov2009introduction} gives a universal positive lower bound on the worst-case probability of decoding the wrong codeword. Given any estimator $\widehat f$, decode by nearest neighbor in $L^2(\mu)$; whenever $\norm{\widehat f-F_z}_2<a_1\lambda/2$, the decoded index is $z$. Markov's inequality yields $\sup_z\E_z\norm{\widehat f-F_z}_2^2\ge c\lambda^2$, and taking the infimum over estimators proves the theorem.

\subsection{Proof of Corollary~\ref{cor:budgets}}
\emph{The general statement.} We show that $C_A=e^3$ and $C_A'=e^4$ suffice. Since $c_M\le1$, $d\le L$, and $\log m\le w$, we have $M_d\le L^2w^3$. Also $\Pi_d\le4w(2w^2)^{d-1}=2^{d+1}w^{2d-1}$. Hence $M_d\Pi_d\le2^{d+1}L^2w^{2d+2}$. If $A\ge e^3Lw^{(2d+2)/L}$, then
\[
  \Big(\frac AL\Big)^L\ge e^{3L}w^{2d+2}\ge2^{L+1}L^2w^{2d+2}\ge M_d\Pi_d ,
\]
because $(L+1)\log2+2\log L\le3L$ for $L\ge1$. Thus $\Lambda_d(A)\ge1$. With $B=\sigma=1$, Theorem~\ref{thm:general-lower} gives $\Risk_n^*(A,1,1)\ge c\min\{1,\sqrt{M_d/n},1\}^2=c\min\{1,M_d/n\}$. The minimax risk is nondecreasing in $A$, since the classes are nested.

\emph{Lower bounds on $M_d$.} If $d,m\ge D_0$ are large enough that $c_Mm^2d^2\log m\ge2$, then $M_d\ge\frac{c_M}2m^2d^2\log m$. For $w\ge3$, $m\ge w/2$ and $\log m\ge\frac12\log w$, so $M_d\ge\frac{c_M}{16}d^2w^2\log w$.

\emph{Item (a).} Take $d=D$. Then $(2d+2)/L=2$, and $d=L-1\ge L/2$ gives $M_D\ge\frac{c_M}{64}L^2w^2\log w$. We need $L\ge D_0+1$.

\emph{Item (b).} Take $d=\lfloor\varepsilon L/2\rfloor-1$. Then $2d+2\le\varepsilon L$, so $e^3Lw^\varepsilon\ge e^3Lw^{(2d+2)/L}$. Also $d\le D$ because $\varepsilon\le2$. If $L\ge2(D_0+2)/\varepsilon$ and $\varepsilon L\ge8$, then $d\ge D_0$ and $d\ge\varepsilon L/4$, so $M_d\ge\frac{c_M}{256}\varepsilon^2L^2w^2\log w$. Since $\frac{c_M}{256}\varepsilon^2\le1$, $\min\{1,M_d/n\}\ge\frac{c_M}{256}\varepsilon^2\min\{1,L^2w^2\log w/n\}$.

\emph{Item (c).} Take $d=\lfloor L/(2\log w)\rfloor-1$. Then $(2d+2)\log w\le L$, so $w^{(2d+2)/L}\le e$ and $e^4L\ge e^3Lw^{(2d+2)/L}$. If $L\ge\max\{2(D_0+2),8\}\log w$, then $D_0\le d\le D$ and $d\ge L/(4\log w)$, so $M_d\ge\frac{c_M}{256}L^2w^2/\log w$.

\emph{Upper bounds.} By Theorem~\ref{thm:upper} with $B=\sigma=1$, $\Risk_n^*(A,1,1)\le C\min\{1,L^2w^2\log(2Lw)\log(en)/n\}$ for every $A$.

\subsection{The normalization of \texorpdfstring{\citet{ganguli2026dichotomy}}{Ganguli and Constantinescu}}
Let $A=B=R$, $c_\sigma R\le\sigma\le C_\sigma R$, and $n\ge M_d$. If $\Lambda_d(R)\ge R$, the three amplitudes in Theorem~\ref{thm:general-lower} are at least $R$, $c_\sigma R\sqrt{M_d/n}$, and $R\ge R\sqrt{M_d/n}$, so $\Risk_n^*(R,R,\sigma)\ge c\min\{1,c_\sigma\}^2R^2M_d/n$. By the argument for the general statement above, $\Lambda_d(R)\ge R$ holds, for example, when $(R/L)^L\ge R\,2^{d+1}L^2w^{2d+2}$.

%% file: sections/F_variation.tex
\section{Derivative variation and the cost of constants}
\label{app:variation}

Throughout this appendix, paths are continuous and piecewise affine on $[0,2]$, and $J$ is defined by~\eqref{eq:J}. Equivalently, $J(g)$ is the total variation on $\R$ of the derivative of the extension of $g$ that is constant outside $[0,2]$. $J$ is a seminorm vanishing on constants, which is the intrinsic variability of $g'$ in the sense of \citet{huster2018limitations}, restricted to $[0,2]$.

\begin{lemma}[Lattice inequality]
\label{lem:lattice}
For continuous piecewise-affine $u,v$ on $[0,2]$, $J(\max\{u,v\})\le J(u)+J(v)$.
\end{lemma}

\begin{proof}
Replace $u,v$ by their constant extensions $\bar u,\bar v$ to $\R$; the constant extension of $\max\{u,v\}$ is $\max\{\bar u,\bar v\}$. Since $\bar u-\bar v$ is piecewise affine with finitely many pieces, so is $\max\{\bar u,\bar v\}$. For such a function $g$, the right derivative $\bar g'$ is a step function with finitely many jumps that vanishes outside $[0,2]$, and $J(g)=\TV(\bar g';\R)$ is the sum of the absolute values of its jumps, including the jumps at $0$ and $2$. The jumps sum to zero, so $J(g)=2\,\mathrm{Down}(g)$, where $\mathrm{Down}(g)=\sum_x(\bar g'(x-)-\bar g'(x+))_+$ is the total downward jump. It therefore suffices to show that, at every point $x$, the downward jump of $m=\max\{\bar u,\bar v\}$ is at most the sum of those of $\bar u$ and $\bar v$.

If $\bar u(x)\ne\bar v(x)$, then $m$ coincides with one of the two functions near $x$, and the claim is immediate. If $\bar u(x)=\bar v(x)$, let $a,b$ be the left and right derivatives of $\bar u$ at $x$, and $c,e$ those of $\bar v$. For small $h>0$, $m(x-h)=m(x)-h\min\{a,c\}$ and $m(x+h)=m(x)+h\max\{b,e\}$, so the downward jump of $m$ at $x$ is
\[
  \big(\min\{a,c\}-\max\{b,e\}\big)_+\le\min\{(a-b)_+,(c-e)_+\}.
\]
This covers every configuration with $\bar u(x)=\bar v(x)$: transverse crossings ($a=b$, $c=e$, so $m$ has only an upward jump), tangential contacts, common breakpoints, and the endpoints of intervals on which $\bar u=\bar v$. Summing over $x$ gives $\mathrm{Down}(m)\le\mathrm{Down}(u)+\mathrm{Down}(v)$, and multiplying by two proves the lemma.
\end{proof}

\begin{lemma}[One-block bounds]
\label{lem:growth}
Let $s$ be a block~\eqref{eq:rbv-block} with cost $a=\norm s_{\RBV}$ and $b=\norm{s(0)}_1$. Then $\Lip_{\ell^2\to\ell^2}(s)\le a-b$ and $\norm{s(z)}_2\le b+(a-b)\norm z_2\le a\max\{1,\norm z_2\}$. Moreover, Lemma~\ref{lem:anchored} holds.
\end{lemma}

\begin{proof}
Let $\Delta_{jr}=s_j(e_r)-s_j(0)$. Since $\relu$ is nondecreasing and $1$-Lipschitz, for each atom $k$ and coordinate $r$ there is $\alpha_{kr}\in[0,1]$ with $\relu(w_{kr}-b_k)-\relu(-b_k)=\alpha_{kr}w_{kr}$ (take $\alpha_{kr}=0$ if $w_{kr}=0$). Evaluating~\eqref{eq:rbv-block} at $0$ and $e_r$ gives $\Delta_{jr}=C_{jr}+\sum_kv_{kj}\alpha_{kr}w_{kr}$. Substituting $C=\Delta-\sum_kv_k(\alpha_k\odot w_k)^\top$ and $c_0=s(0)-\sum_kv_k\relu(-b_k)$ into~\eqref{eq:rbv-block} yields the exact identity
\begin{equation}
  s(z)=s(0)+\Delta z+\sum_kv_k\,\phi_k(z),\qquad
  \phi_k(z)=\relu(w_k^\top z-b_k)-\relu(-b_k)-\sum_r\alpha_{kr}w_{kr}z_r .
  \label{eq:anchored}
\end{equation}
Thus every ridge is corrected by its own secant slopes along the coordinate axes, and the skip matrix $C$ disappears in favor of the anchor matrix $\Delta$, which is what the cost charges.

\emph{Lipschitz and growth bounds.} Wherever $\phi_k$ is differentiable, $\partial_r\phi_k(z)=(\mathbf 1\{w_k^\top z>b_k\}-\alpha_{kr})w_{kr}$, and $|\mathbf 1\{\cdot\}-\alpha_{kr}|\le1$, so $\norm{\nabla\phi_k}_2\le\norm{w_k}_2$. The rows of the Jacobian of $s$ are $\nabla s_j=\Delta_{j\cdot}+\sum_kv_{kj}\nabla\phi_k$, so its operator norm is at most
\[
  \sum_j\norm{\nabla s_j}_2\le\sum_{j,r}|\Delta_{jr}|+\sum_k\norm{v_k}_1\norm{w_k}_2=a-b .
\]
Since $s$ is piecewise affine, integrating along segments gives $\Lip(s)\le a-b$, and then $\norm{s(z)}_2\le\norm{s(0)}_2+(a-b)\norm z_2\le b+(a-b)\norm z_2$.

\emph{Comparison.} Let $g:[0,2]\to\R^{d_{\rm in}}$ be a continuous piecewise-affine path. Writing $\relu(y)=\max\{y,0\}$ in~\eqref{eq:anchored},
\[
  \phi_k(z)=\max\Big\{\sum_r(1-\alpha_{kr})w_{kr}z_r-b_k,\ -\sum_r\alpha_{kr}w_{kr}z_r\Big\}-\relu(-b_k).
\]
Both branches composed with $g$ are linear combinations of the $g_r$ plus constants. By Lemma~\ref{lem:lattice}, the seminorm property, and invariance of $J$ under adding constants,
\[
  J(\phi_k\circ g)\le\sum_r(1-\alpha_{kr})|w_{kr}|J(g_r)+\sum_r\alpha_{kr}|w_{kr}|J(g_r)=\sum_r|w_{kr}|J(g_r)\le\norm{w_k}_2\norm{\mathbf J(g)}_2 ,
\]
the last step by Cauchy--Schwarz. The coefficients $1-\alpha_{kr}$ and $\alpha_{kr}$ are complementary, so splitting the ridge into two branches costs nothing. By~\eqref{eq:anchored}, $J(s_j\circ g)\le\sum_r|\Delta_{jr}|J(g_r)+\sum_k|v_{kj}|\,\norm{w_k}_2\norm{\mathbf J(g)}_2$. Bound $\norm{\mathbf J(s\circ g)}_2$ by the $\ell^1$ norm $\sum_jJ(s_j\circ g)$, sum over $j$, and use $J(g_r)\le\norm{\mathbf J(g)}_2$:
\[
  \norm{\mathbf J(s\circ g)}_2\le\Big(\sum_{j,r}|\Delta_{jr}|+\sum_k\norm{v_k}_1\norm{w_k}_2\Big)\norm{\mathbf J(g)}_2=(a-b)\norm{\mathbf J(g)}_2 .
\]
No width factor arises because each norm is paired with its dual in the cost. Output coordinates are summed in $\ell^1$, which is exactly how $\norm{v_k}_1$ and the anchor terms sum over $j$. Inputs enter through $\langle w_k,\cdot\rangle$, which Cauchy--Schwarz bounds by $\norm{w_k}_2$ against the $\ell^2$ norm of $\mathbf J(g)$. Passing from the $\ell^1$ output bound to the $\ell^2$ input of the next block loses nothing, since $\norm\cdot_2\le\norm\cdot_1$.
\end{proof}

\subsection{Proof of Proposition~\ref{prop:variation}}
Let $f\in\cC_{L,w}(A,B)$, $\eps>0$, and let $f=s_L\circ\cdots\circ s_1$ be a representation with block costs $a_\ell$ and $\sum_\ell a_\ell\le A+\eps$. Put $g_0(t)=t$ and $g_\ell=s_\ell\circ g_{\ell-1}$ on $[0,2]$. Each $g_\ell$ is a continuous piecewise-affine path, and $g_L=f$. The constant extension of $g_0$ has derivative $\mathbf 1_{(0,2)}$, so $J(g_0)=0+1+1=2$. Applying Lemma~\ref{lem:anchored} to each block and dropping the nonnegative terms $\norm{s_\ell(0)}_1$,
\[
  J(f)=\norm{\mathbf J(g_L)}_2\le2\prod_{\ell=1}^La_\ell\le2\Big(\frac{A+\eps}L\Big)^L
\]
by AM--GM, and $\eps\to0$ gives $J(f)\le2\kappa$. The derivative of the constant extension of $f$ vanishes outside $[0,2]$, so at every point its absolute value is at most its variation on either side, and $\sup|f'|\le J(f)/2\le \kappa$. Finally, the total variation of the measure $f''$ on $(0,2)$ is $J(f)-|f'(0+)|-|f'(2-)|\le2\kappa$.

Members of the class need not satisfy $f(0)=f(2-)$, so these are statements about $f$ on the interval, and they are the only ones used later. The class is contained in a second-variation ball. Equality fails in general, because members of the class have finitely many breakpoints and a fixed architecture.

\subsection{Proof of Lemma~\ref{lem:constants}}
\emph{Envelopes.} Let $f=s_L\circ\cdots\circ s_1$ have block costs $a_\ell$ with $\sum_\ell a_\ell\le A+\eps$, and let $E_\ell$ be the $L^p(\mu;\ell^2)$ norm of the intermediate path after $\ell$ blocks, with $E_0=\norm t_{L^p(\mu)}$. By Lemma~\ref{lem:growth}, $\norm{s_\ell(z)}_2\le b_\ell+(a_\ell-b_\ell)\norm z_2$ with $b_\ell=\norm{s_\ell(0)}_1\le a_\ell$, so Minkowski's inequality in $L^p(\mu)$ gives $E_\ell\le b_\ell+(a_\ell-b_\ell)E_{\ell-1}\le a_\ell\max\{1,E_{\ell-1}\}$. (The pointwise bound $a\max\{1,\norm z_2\}$ alone would not suffice for $p<\infty$.) Iterating,
\[
  E_L\le\max\Big\{\max_{1\le j<L}\prod_{\ell>L-j}a_\ell,\ \max\{1,E_0\}\prod_{\ell=1}^La_\ell\Big\}\le\max\{1,E_0\}\,H_L(A+\eps),
\]
since a product of $j$ numbers with sum at most $S$ is at most $(S/j)^j$. With $p=\infty$ ($E_0=2$) and $p=2$ ($E_0^2=4/3$), letting $\eps\to0$ gives $\norm f_\infty\le2H_L(A)$ and $\norm f_2^2\le\frac43H_L(A)^2$. Evaluating the same recursion at the single point $t=0$ gives $|f(0)|\le H_L(A)$.

\emph{Exact cost of constants.} Let $a\ne0$ and $1\le j\le L$. Use $L-j$ zero blocks (cost $0$), then the constant block $\operatorname{sgn}(a)|a|^{1/j}$ (cost $|a|^{1/j}$), then $j-1$ scalar gains $u\mapsto|a|^{1/j}u$ (cost $|a|^{1/j}$ each). The composition is the constant $a$ at cost $j|a|^{1/j}$, so $\VarCost_{L,w}(a)\le\min_jj|a|^{1/j}$. Conversely, if the constant $a$ has a representation of cost $S$, then $|a|=|f(0)|\le H_L(S)=\max_j(S/j)^j$, so $S\ge j|a|^{1/j}$ for some $j$. In particular, the constants in $\cC_{L,w}(A,B)$ are exactly those with $|a|\le\min\{B,H_L(A)\}$.

%% file: sections/G_upper_bound.tex
\section{Proof of the upper bounds}
\label{app:upper}

\subsection{Pseudodimension and covering numbers}
Let $V$ be the right-hand side of~\eqref{eq:pdim-bound}, so that $\Pdim(\cC_{L,w}(A,B))\le V$ by Lemma~\ref{lem:graph-main}.
A standard pseudodimension covering theorem \citep{haussler1992decision,anthony1999neural} states that for every probability measure $P$ and every $[-B,B]$-valued class $\cF$ of pseudodimension at most $V$,
\begin{equation}
  \log N(\eps,\cF,L^2(P))\le CV\log\frac{CB}{\eps},\qquad 0<\eps\le B.
  \label{eq:cover-full}
\end{equation}

\subsection{Finite-class least squares with Gaussian noise}
\begin{lemma}
\label{lem:finite-oracle}
Let $\cG$ be a finite class of functions $g:\mathcal X\to[-B,B]$, and suppose $Y=f^\star(X)+\xi$ with $|f^\star|\le B$ and $\xi\sim N(0,\sigma^2)$ independent of $X$. Let
\[
  \widehat g\in\arg\min_{g\in\cG}\frac1n\sum_{i=1}^n(Y_i-g(X_i))^2 .
\]
Then
\begin{equation}
  \E\norm{\widehat g-f^\star}_{L^2(P_X)}^2\le
  C\inf_{g\in\cG}\norm{g-f^\star}_{L^2(P_X)}^2
  +C\,\frac{(\sigma^2+B^2)\log(2|\cG|)}{n}.
  \label{eq:finite-oracle-full}
\end{equation}
\end{lemma}

\begin{proof}
For $g\in\cG$, put $d_g=g-f^\star$, $r_g=Pd_g^2$, and
\[
  Z_i(g)=d_g(X_i)^2-2\xi_id_g(X_i),\qquad Z_n(g)=\frac1n\sum_{i=1}^nZ_i(g).
\]
Then $\E Z_i(g)=r_g$, and empirical risk minimization implies $Z_n(\widehat g)\le Z_n(g)$ for every $g\in\cG$.

We verify Bernstein's moment condition $\E|Z_i(g)-r_g|^k\le\frac{k!}2\nu_g^2b^{k-2}$ for $k\ge2$. Write $Z_i(g)-r_g=U+W$ with $U=d_g(X_i)^2-r_g$ and $W=-2\xi_id_g(X_i)$. Since $|d_g|\le2B$, $|U|\le4B^2$ and $\E U^2\le\E d_g^4\le4B^2r_g$, so $\E|U|^k\le4B^2r_g(4B^2)^{k-2}$. Conditionally on $X_i$, $W\sim N(0,4\sigma^2d_g(X_i)^2)$, and $\E|N(0,1)|^k\le(k-1)!!\le k!/2$, so
\[
  \E|W|^k\le\frac{k!}2(2\sigma)^k\E|d_g|^k\le\frac{k!}2\,4\sigma^2r_g\,(4\sigma B)^{k-2}.
\]
Since $\E|U+W|^k\le2^{k-1}(\E|U|^k+\E|W|^k)$, the condition holds with
\begin{equation}
  \nu_g^2=8(\sigma^2+B^2)r_g,\qquad b=8B(B+\sigma)\le12a,\qquad a\defeq\sigma^2+B^2.
  \label{eq:bernstein-variance}
\end{equation}
Bernstein's inequality gives $|Z_n(g)-r_g|\le\sqrt{2\nu_g^2u/n}+bu/n$ with probability at least $1-2e^{-u}$, and $\sqrt{16ar_gu/n}\le r_g/2+8au/n$. Hence, for each $u>0$, with probability at least $1-2e^{-u}$,
\begin{equation}
  \frac12r_g-Ca\frac un\le Z_n(g)\le\frac32r_g+Ca\frac un .
  \label{eq:bernstein-two-sided}
\end{equation}

Apply the lower inequality in~\eqref{eq:bernstein-two-sided} simultaneously to every $g\in\cG$ with $u=t+\log(2|\cG|)$, and apply the upper inequality to a fixed comparator $g_0$ with $u=t$. With probability at least $1-3e^{-t}$,
\begin{align*}
  \frac12r_{\widehat g}
  &\le Z_n(\widehat g)+Ca\frac{t+\log(2|\cG|)}{n}
  \le Z_n(g_0)+Ca\frac{t+\log(2|\cG|)}{n}
  \le\frac32r_{g_0}+Ca\frac{2t+\log(2|\cG|)}{n}.
\end{align*}
Thus
\[
  r_{\widehat g}\le3r_{g_0}+Ca\frac{2t+\log(2|\cG|)}{n}.
\]
Integrating the exponential tail over $t\ge0$ and minimizing over $g_0$ proves~\eqref{eq:finite-oracle-full}.
\end{proof}

\subsection{The architecture term}
With $V$ as above, if $(\sigma^2+B^2)V/n\ge B^2$, the zero estimator has risk at most $B^2$. Otherwise set
\[
  \eps^2=\frac{(\sigma^2+B^2)V}{n}<B^2
\]
and let $\cG\subseteq\cC_{L,w}(A,B)$ be an $\eps$-net in $L^2(\mu)$. By~\eqref{eq:cover-full},
\[
  \log|\cG|\le CV\log\frac{CB}{\eps}\le CV\log(en),
\]
where the last inequality uses $\eps^2\ge B^2V/n$ and $V\ge1$. Lemma~\ref{lem:finite-oracle}, applied to least squares over $\cG$, gives
\[
  \Risk_n^*(A,B,\sigma)\le C\eps^2+C\,\frac{(\sigma^2+B^2)V\log(en)}{n}
  \le C\,\frac{(\sigma^2+B^2)V\log(en)}{n}.
\]
This is the last term of Theorem~\ref{thm:upper}. Since every member of the class satisfies $\norm f_\infty\le2H_L(A)$ (Lemma~\ref{lem:constants}), $B$ may be replaced by $\min\{B,2H_L(A)\}$ in this term.

\subsection{The budget-sensitive term}

After centering, the class lies in a bounded part of a second-variation ball, and its entropy follows from the entropy of convex functions. We use the one-dimensional case of \citet[Theorem~3.2]{guntuboyina2013covering}, which extends Bronshtein's theorem. There are universal $C_{\rm GS},\eps_0>0$ such that the convex functions on an interval $[a,b]$ that are bounded in absolute value by $U$ and are $\Gamma$-Lipschitz have an $\eps$-cover in the sup norm of log-cardinality at most $C_{\rm GS}\sqrt{(U+\Gamma(b-a))/\eps}$, for all $0<\eps\le\eps_0(U+\Gamma(b-a))$.

\begin{lemma}[Entropy of centered second-variation balls]
\label{lem:cover}
Let $0<S\le \kappa$ and let $\mathcal G$ be the set of continuous piecewise-affine $g$ on $[0,2]$ with $\int g\dd\mu=0$, $J(g)\le2\kappa$, and $\norm g_\infty\le S$. For every integer $k\ge1$ there is a finite set $\cH_k$ of functions with $\int h\dd\mu=0$ and $\norm h_\infty\le2S$ such that
\[
  \log(2|\cH_k|)\le Ck,\qquad
  \sup_{g\in\mathcal G}\min_{h\in\cH_k}\norm{g-h}_\infty\le\frac{2\kappa}{k^2} .
\]
\end{lemma}

\begin{proof}
Let $g\in\mathcal G$ and let $\nu=g''$ be its distributional second derivative on $(0,2)$. Since $g$ is piecewise affine, $\nu$ is a finite combination of Dirac masses at the interior breakpoints, weighted by the slope jumps, and $|\nu|((0,2))=J(g)-|g'(0+)|-|g'(2-)|\le2\kappa$. Let $\nu=\nu_+-\nu_-$ be its Jordan decomposition. Integrating $g'(t)=g'(0+)+\nu((0,t])$ gives, for $t\in[0,2]$,
\[
  g(t)=g(0)+g'(0+)\,t+u_+(t)-u_-(t),\qquad u_\pm(t)=\int_{(0,2)}(t-v)_+\,\nu_\pm(\dd v).
\]
Each $u_\pm$ is a nonnegative mixture of the convex functions $(t-v)_+$. So it is convex, takes values in $[0,2\nu_\pm((0,2))]\subseteq[0,4\kappa]$, and is $2\kappa$-Lipschitz. Moreover $|g(0)|\le S\le \kappa$. Also $|g'(0+)|\le J(g)/2\le \kappa$, because the derivative of the constant extension of $g$ starts at zero.

Fix $0<\eps\le \kappa$. With $U=4\kappa$, $\Gamma=2\kappa$ and $[a,b]=[0,2]$, we have $U+\Gamma(b-a)=8\kappa$. The cited theorem gives an $(\eps/4)$-cover $\mathcal U_\eps$ of these convex functions in the sup norm with $\log|\mathcal U_\eps|\le C\sqrt{\kappa/\eps}$. This holds when $\eps/4\le8\eps_0\kappa$. For larger $\eps\le \kappa$, the cover at accuracy $8\eps_0\kappa$ serves, with log-cardinality $O(1)$. Choose $\phi_\pm\in\mathcal U_\eps$ within $\eps/4$ of $u_\pm$. Choose $\beta_0,\beta_1$ from the grid $(\eps/8)\mathbb Z\cap[-\kappa,\kappa]$ with $|\beta_0-g(0)|\le\eps/16$ and $|\beta_1-g'(0+)|\le\eps/16$. Since $|t|\le2$,
\[
  \norm{g-(\beta_0+\beta_1t+\phi_+-\phi_-)}_\infty\le\frac\eps{16}+\frac{2\eps}{16}+\frac\eps4+\frac\eps4<\eps .
\]
The number of such sums has logarithm at most $2\log|\mathcal U_\eps|+2\log(16\kappa/\eps+1)\le C\sqrt{\kappa/\eps}$, because $\log(16x+1)\le C\sqrt x$ for $x\ge1$.

For a function $\psi$, let $\bar\psi=\max\{-S,\min\{\psi,S\}\}$ and $\Pi\psi=\bar\psi-\int\bar\psi\dd\mu$. Clipping is pointwise $1$-Lipschitz and $|g|\le S$, so $\norm{\bar\psi-g}_\infty\le\norm{\psi-g}_\infty$. Since $\int g\dd\mu=0$, $|\int\bar\psi\dd\mu|=|\int(\bar\psi-g)\dd\mu|\le\norm{\bar\psi-g}_\infty$, so $\norm{\Pi\psi-g}_\infty\le2\norm{\psi-g}_\infty$. Moreover, $\Pi\psi$ has mean zero and $\norm{\Pi\psi}_\infty\le2\norm{\bar\psi}_\infty\le2S$. Let $\cH_k$ be the image under $\Pi$ of all the sums above, with $\eps=\kappa/k^2\le \kappa$. Then $\log(2|\cH_k|)\le C\sqrt{\kappa/\eps}+1\le Ck$, and every $g\in\mathcal G$ is within $2\eps=2\kappa/k^2$ of $\cH_k$.
\end{proof}

We now prove that, for all $L\ge2$, $w\ge1$, $n\ge2$ and $A,B,\sigma>0$,
\begin{equation}
  \Risk_n^*(A,B,\sigma)\le q+C\min\Big\{\kappa^2,\ \kappa^{2/5}\Big(\frac{\Dsig}{n}\Big)^{4/5}\Big\}.
  \label{eq:budget-term}
\end{equation}
Fix $f\in\cC_{L,w}(A,B)$ and write $f=\theta+g$ with $\theta=\int f\dd\mu$. By Proposition~\ref{prop:variation}, $J(g)=J(f)\le2\kappa$ and $\Lip(g)\le \kappa$. Since $g$ has mean zero, for every $t\in[0,2]$,
\[
  |g(t)|=\Big|\int(g(t)-g(t'))\,\mu(\dd t')\Big|\le \kappa\int_0^2|t-t'|\,\frac{\dd t'}2\le \kappa,
\]
because the last integral is $(t^2+(2-t)^2)/4\le1$. Also $\norm g_\infty\le\norm f_\infty+|\theta|\le2B$. Hence $g\in\mathcal G$ with $S=\min\{\kappa,2B\}$, and $S^2\le4\min\{\kappa^2,B^2\}$. In particular, $\Var_\mu(f)=\norm g_2^2\le S^2$.

\emph{The sample mean.} The estimator $\overline Y$ satisfies $\E(\overline Y-\theta)^2=(\sigma^2+\Var_\mu(f))/n$. Hence
\begin{equation}
  \E\norm{\overline Y-f}_{L^2(\mu)}^2=\frac{\sigma^2+\Var_\mu(f)}n+\Var_\mu(f)\le q+2\kappa^2.
  \label{eq:mean-only}
\end{equation}

\emph{Paired least squares for the centered part.} Because $\theta$ is unknown, fitting $Y_i$ by $\theta'+h(T_i)$ would require a net over $\theta'$ as well. Taking differences of observations removes $\theta$ exactly. Let $n'=\lfloor n/2\rfloor\ge n/3$ and, for $i\le n'$,
\[
  X_i=(T_{2i},T_{2i-1}),\quad Z_i=Y_{2i}-Y_{2i-1}=G(X_i)+\xi_i',\quad G(t,t')=g(t)-g(t'),\quad \xi_i'=\xi_{2i}-\xi_{2i-1}.
\]
The pairs $(X_i,Z_i)$ are independent and identically distributed, since they use disjoint observations. Each $X_i$ has law $\mu\otimes\mu$, and $\xi_i'\sim N(0,2\sigma^2)$ is independent of $X_i$. This is the model of Lemma~\ref{lem:finite-oracle} on $\mathcal X=[0,2)^2$, with regression function $G$, noise variance $2\sigma^2$, and design $\mu\otimes\mu$. Let
\[
  \widehat h\in\arg\min_{h\in\cH_k}\sum_{i=1}^{n'}\big(Z_i-h(T_{2i})+h(T_{2i-1})\big)^2 .
\]
This is least squares over the finite class $\{H_h(t,t')=h(t)-h(t'):h\in\cH_k\}$. The function $G$ and every $H_h$ are bounded by $4S$. Lemma~\ref{lem:finite-oracle} with bound $4S$ and noise variance $2\sigma^2$ gives
\[
  \E\norm{H_{\widehat h}-G}_{L^2(\mu\otimes\mu)}^2\le C\min_{h\in\cH_k}\norm{H_h-G}_{L^2(\mu\otimes\mu)}^2+C\frac{(\sigma^2+S^2)\log(2|\cH_k|)}{n'}
  \le C\frac{\kappa^2}{k^4}+C\frac{\Dsig k}{n}.
\]
The last step uses $\norm{H_h-G}_\infty\le2\norm{h-g}_\infty$, Lemma~\ref{lem:cover}, $n'\ge n/3$, and $\sigma^2+S^2\le4\Dsig$. If $u$ has $\int u\dd\mu=0$ and $T,T'\sim\mu$ are independent, then $\E(u(T)-u(T'))^2=2\norm u_{L^2(\mu)}^2$. For every $h\in\cH_k$, the function $u=h-g$ has mean zero and $H_h-G=u(t)-u(t')$. Applying the identity with $h=\widehat h$, conditionally on the sample, gives
\begin{equation}
  \E\norm{\widehat h-g}_{L^2(\mu)}^2\le C\frac{\kappa^2}{k^4}+C\frac{\Dsig k}{n}.
  \label{eq:paired-risk}
\end{equation}

\emph{Combination.} Let $\widehat f=\overline Y+\widehat h$. Then $\widehat f-f=(\overline Y-\theta)+(\widehat h-g)$. Every element of $\cH_k$ has mean zero, so $\int(\widehat h-g)\dd\mu=0$ for every realization. The cross term $2(\overline Y-\theta)\int(\widehat h-g)\dd\mu$ therefore vanishes identically, even though $\overline Y$ and $\widehat h$ are dependent. By~\eqref{eq:paired-risk},
\[
  \E\norm{\widehat f-f}_{L^2(\mu)}^2=\E(\overline Y-\theta)^2+\E\norm{\widehat h-g}_2^2
  \le q+\frac{\Var_\mu(f)}{n}+C\frac{\kappa^2}{k^4}+C\frac{\Dsig k}{n}.
\]

\emph{Choice of $k$.} The last two terms are the approximation error and the complexity of the net. They balance at $k_*=(\kappa^2n/\Dsig)^{1/5}$, and $\kappa^{2/5}(\Dsig/n)^{4/5}\ge \kappa^2$ if and only if $k_*\le1$.

If $k_*\ge1$, take the integer $k=\lceil k_*\rceil\in[k_*,2k_*]$. Then
\[
  \frac{\kappa^2}{k^4}\le\frac{\kappa^2}{k_*^4}=\kappa^{2/5}\Big(\frac{\Dsig}n\Big)^{4/5},\qquad
  \frac{\Dsig k}n\le\frac{2\Dsig k_*}n=2\kappa^{2/5}\Big(\frac{\Dsig}n\Big)^{4/5}.
\]
Also, $\Var_\mu(f)/n\le S^2/n\le4\Dsig/n\le4\kappa^{2/5}(\Dsig/n)^{4/5}$. The last step is equivalent to $(\Dsig/n)^{1/5}\le \kappa^{2/5}$, that is, to $k_*\ge1$. So the risk of $\widehat f$ is at most $q+C\kappa^{2/5}(\Dsig/n)^{4/5}=q+C\min\{\kappa^2,\kappa^{2/5}(\Dsig/n)^{4/5}\}$.

If $k_*<1$, the minimum in~\eqref{eq:budget-term} equals $\kappa^2$, and the sample mean attains it by~\eqref{eq:mean-only}. This proves~\eqref{eq:budget-term}.

The same entropy bound limits every packing of the class, which is used in Remark~\ref{rem:regimes}.

\begin{corollary}[Packings of the class]
\label{cor:packing-bound}
If $\cF\subseteq\cC_{L,w}(A,B)$ is finite and $\norm{f-f'}_{L^2(\mu)}\ge\lambda>0$ for all distinct $f,f'\in\cF$, then
\[
  \log|\cF|\le C\Big(1+\sqrt{\kappa/\lambda}\Big)+\log\Big(1+\frac{8B}{\lambda}\Big).
\]
\end{corollary}

\begin{proof}
Write $f=\theta_f+g_f$ with $\theta_f=\int f\dd\mu\in[-B,B]$. As shown above, $g_f\in\mathcal G$ with $S=\min\{\kappa,2B\}$. Let $k=\lceil4\sqrt{\kappa/\lambda}\,\rceil$, so that $2\kappa/k^2\le\lambda/8$, and take $\cH_k$ from Lemma~\ref{lem:cover}. Let $\mathcal B$ be a grid in $[-B,B]$ of mesh $\lambda/4$ with at most $1+8B/\lambda$ points. To each $f\in\cF$ assign $h\in\cH_k$ with $\norm{g_f-h}_\infty\le\lambda/8$ and $\beta\in\mathcal B$ with $|\theta_f-\beta|\le\lambda/8$, so that $\norm{f-(\beta+h)}_\infty\le\lambda/4$. If two distinct members of $\cF$ received the same pair, their $L^2(\mu)$ distance would be at most their sup distance, which is at most $\lambda/2<\lambda$. Hence $|\cF|\le|\cH_k|\,|\mathcal B|$, and $\log|\cH_k|\le Ck\le C(1+\sqrt{\kappa/\lambda})$.
\end{proof}

\subsection{Proof of Theorem~\ref{thm:upper}}
The zero estimator has risk $\norm f_2^2\le\min\{B^2,\frac43H_L(A)^2\}$ by Lemma~\ref{lem:constants}. Together with the two previous subsections, this shows that the minimax risk is at most a constant times each of the four terms, and hence at most $C$ times their minimum.

%% file: sections/H_low_budget.tex
\section{The low-budget regime and the critical budget}
\label{app:low}

\subsection{Two lower bounds}

\begin{lemma}[Constants]
\label{lem:const-lower}
For all $L,w,A,B,\sigma,n$, $\ \Risk_n^*(A,B,\sigma)\ge\frac1{16}\min\{B^2,H_L(A)^2,q\}$.
\end{lemma}

\begin{proof}
Let $a=\frac12\min\{\min\{B,H_L(A)\},\sigma/\sqrt n\}$. By Lemma~\ref{lem:constants}, the constants $\pm a$ belong to the class. The $n$-sample laws have $\KL=2na^2/\sigma^2\le\frac12$, hence total variation at most $\frac12$. For any estimator, Jensen's inequality gives $\norm{\widehat f-c}_2^2\ge(\int\widehat f\dd\mu-c)^2$, and a wrong sign of $\int\widehat f\dd\mu$ costs at least $a^2$. By Le Cam's two-point bound the worst-case risk is at least $a^2(1-\frac12)/2=a^2/4$.
\end{proof}

\begin{lemma}[Spline hypercubes]
\label{lem:spline}
Let $w\ge3$ and $W=\lfloor w/3\rfloor$. Then
\[
  \Risk_n^*(A,B,\sigma)\ge c\max_{1\le k\le W}\min\Big\{B^2,\ \frac{\kappa^2}{k^4},\ qk\Big\}.
\]
\end{lemma}

\begin{figure}[ht]
\centering
\input{figures/spline_packing}
\caption{The spline hypercube of Lemma~\ref{lem:spline}, drawn for $k=4$. (a)~Tents $\phi_j$ of height one and base $h=1/(2k)$ with disjoint supports in $[\frac14,\frac34]$. The function $f_\omega$ keeps the tents with $\omega_j=1$ (solid) and omits those with $\omega_j=0$ (dotted), so the $2^k$ functions $f_\omega$ form a hypercube with orthogonal directions. (b)~One block with $3k$ atoms realizes $a^{-(L-1)}f_\omega$, and $L-1$ scalar gains $u\mapsto au$ restore $f_\omega$. With $a=A/L$ the total cost is at most $A$ whenever $\lambda\le\kappa/(16k^2)$.}
\label{fig:spline}
\end{figure}

\begin{proof}
Fix $1\le k\le W$ and let $h=1/(2k)$, $a_j=\frac14+(j-1)h$. The tents $\phi_j(t)=\frac2h\{\relu(t-a_j)-2\relu(t-a_j-h/2)+\relu(t-a_j-h)\}$, $1\le j\le k$, have height one and disjoint supports in $[\frac14,\frac34]$. Each satisfies $\phi_j(0)=\phi_j(1)=0$, has representation cost $\frac2h(1+2+1)=16k$, and $\norm{\phi_j}_{L^2(\mu)}^2=h/6=1/(12k)$. For $\omega\in\{0,1\}^k$ let $f_\omega=\lambda\sum_j\omega_j\phi_j$, a single block with $3k\le w$ atoms, cost at most $16\lambda k^2$, and $\norm{f_\omega}_\infty\le\lambda$.

\emph{Membership.} Put $a=A/L$, $s_1=a^{-(L-1)}f_\omega$, and $s_\ell(u)=au$ for $2\le\ell\le L$. Then $s_L\circ\cdots\circ s_1=f_\omega$, all intermediate dimensions are one, and each block costs at most $a$ provided $\lambda\le a^L/(16k^2)=\kappa/(16k^2)$. So $f_\omega\in\cC_{L,w}(A,B)$ when also $\lambda\le B$.

\emph{Testing.} Let $\lambda^2=\min\{B^2,\kappa^2/(256k^4),\sigma^2k/n\}$. Strings differing in one bit give $\KL=n\lambda^2/(24\sigma^2k)\le1/24$, hence total variation at most $1/\sqrt{48}$. Given an estimator $\widehat f$, let $\widehat\omega_j=\mathbf 1\{\langle\widehat f,\phi_j\rangle_\mu\ge\frac\lambda2\norm{\phi_j}^2_2\}$. Since the $\phi_j$ are orthogonal, $\norm{\widehat f-f_\omega}_2^2\ge\frac{\lambda^2}{48k}d_H(\widehat\omega,\omega)$. Assouad's lemma \citep[Theorem~2.12]{tsybakov2009introduction} gives $\max_\omega\E\,d_H(\widehat\omega,\omega)\ge\frac k2(1-1/\sqrt{48})$, so the risk is at least $c\lambda^2\ge c\min\{B^2,\kappa^2/k^4,qk\}$.
\end{proof}

\subsection{Proof of Theorem~\ref{thm:low}}
\emph{Upper bound.} This is the only place where the noise condition is used. Since $\min\{\kappa,B\}\le C_1\sigma$, we have $\Dsig=\sigma^2+\min\{\kappa,B\}^2\le(1+C_1^2)\sigma^2$, and so $\kappa^{2/5}(\Dsig/n)^{4/5}\le(1+C_1^2)^{4/5}\kappa^{2/5}q^{4/5}$. Bounding the minimum in the third term of Theorem~\ref{thm:upper} by its second entry and keeping the first two terms gives
\[
  \Risk_n^*(A,B,\sigma)\le C\min\Big\{B^2,\ H_L(A)^2,\ q+(1+C_1^2)^{4/5}\kappa^{2/5}q^{4/5}\Big\}\le C(1+C_1^2)\,\Phi .
\]

\emph{Lower bound.} Put $\tau=\kappa^{2/5}q^{4/5}$, so that $\Phi=\min\{B^2,H_L(A)^2,q+\tau\}$. Note that $\tau\le q$ if and only if $\kappa^2\le q$.

\emph{Case $\kappa^2\le q$.} Then $q+\tau\le2q$, so $\Phi\le2\min\{B^2,H_L(A)^2,q\}$. Lemma~\ref{lem:const-lower} gives $\Risk_n^*\ge\frac1{16}\min\{B^2,H_L(A)^2,q\}\ge\Phi/32$. Here the constant functions carry the whole lower bound, and the cap $H_L(A)$ enters through them.

\emph{Case $\kappa^2>q$.} Put $k_*=(\kappa^2/q)^{1/5}$. Then $k_*>1$, and $k_*\le W=\lfloor w/3\rfloor$ by the width assumption. Take $k=\lfloor k_*\rfloor$. Since $W$ is an integer, $1\le k\le W$, and $k\ge k_*/2$ because $k_*>1$. By the definition of $k_*$, $qk_*=\kappa^2/k_*^4=\tau$. Hence
\[
  qk\ge\frac{qk_*}2=\frac\tau2,\qquad\frac{\kappa^2}{k^4}\ge\frac{\kappa^2}{k_*^4}=\tau ,
\]
and Lemma~\ref{lem:spline} gives $\Risk_n^*\ge c\min\{B^2,\tau,\tau/2\}\ge\frac c2\min\{B^2,\tau\}$. In this case $q<\tau$, so $\Phi\le\min\{B^2,q+\tau\}\le2\min\{B^2,\tau\}$, and therefore $\Risk_n^*\ge\frac c4\Phi$. The cap $H_L(A)$ is inactive here, since $\tau<\kappa^2\le H_L(A)^2$. The cap $B$ enters through the amplitude restriction $\lambda\le B$ in Lemma~\ref{lem:spline}. Rounding $k_*$ down loses at most a factor $2$ in $qk$ and only increases $\kappa^2/k^4$.

In both cases $\Risk_n^*\ge c\Phi$ with a universal $c$. The lower bound uses neither the noise condition nor $n\ge2$.

\subsection{Very small noise}

\begin{proposition}[Interpolation regime]
\label{prop:smallnoise}
Let $w\ge3$, $n\ge2$, $B\ge \kappa/n$, and $\sigma^2\le \kappa^2/n^3$. Then $c\kappa^2/n^3\le\Risk_n^*(A,B,\sigma)\le C\kappa^2/n^3$.
\end{proposition}

\begin{proof}
\emph{Lower bound.} Let $h=1/(4n)$ and let $\phi$ be a tent of height $a=\kappa h/8$ and base $[x_0,x_0+h]\subset(\frac14,\frac34)$, written with three atoms as in Lemma~\ref{lem:spline}. Its single-block cost is $8a/h=\kappa$. The scalar-gain embedding in the proof of Lemma~\ref{lem:spline} therefore places $\phi$ in $\cC_{L,w}(A,B)$, because $3\le w$ and $\norm\phi_\infty=a=\kappa/(32n)\le B$. Let $E$ be the event that no design point falls in the base of $\phi$. Then $\Pr(E)=(1-h/2)^n\ge1-nh/2=7/8$. Off the base, $\phi$ and $0$ coincide. So on $E$ the samples generated by $f^\star=0$ and by $f^\star=\phi$ are identical, for the same design and noise. Hence the total variation distance between the two sample laws is at most $\Pr(E^c)\le1/8$. Since $\norm{\widehat f}_2^2+\norm{\widehat f-\phi}_2^2\ge\frac12\norm\phi_2^2$, the two-point argument gives
\[
  \max_{f^\star\in\{0,\phi\}}\E_{f^\star}\norm{\widehat f-f^\star}_2^2\ge\frac14\norm\phi_2^2\Big(1-\frac18\Big)=\frac7{32}\cdot\frac{a^2h}6\ge c\frac{\kappa^2}{n^3}.
\]
This bound holds for every $\sigma>0$.

\emph{Upper bound.} Fix $f\in\cC_{L,w}(A,B)$. Members of the class need not be periodic, so we interpolate on the interval. Let $T_{(1)}<\cdots<T_{(n)}$ be the ordered design points (distinct almost surely), with responses $Y_{(i)}$. Let $\mathcal P$ be the linear operator that maps values $(y_i)$ at these points to the function that is affine on each interior gap $[T_{(i)},T_{(i+1)}]$, equal to $y_1$ on $[0,T_{(1)}]$, and equal to $y_n$ on $[T_{(n)},2)$. Let $\widehat f=\mathcal P(Y_{(i)})$. No clipping is needed. By linearity, $\widehat f-f=e+\zeta$ with $e=\mathcal P(f(T_{(i)}))-f$ and $\zeta=\mathcal P(\xi_{(i)})$. Conditionally on the design, $e$ is deterministic and $\zeta$ is centered, so $\E\norm{\widehat f-f}_2^2=\E\norm e_2^2+\E\norm\zeta_2^2$.

\emph{Noise.} At a point $t$ of an interior gap at relative position $u\in[0,1]$, $\zeta(t)=(1-u)\xi_{(i)}+u\xi_{(i+1)}$ has conditional variance $\sigma^2((1-u)^2+u^2)\le\sigma^2$. On the two end pieces, the variance is $\sigma^2$. Hence $\E\norm\zeta_2^2\le\sigma^2\le \kappa^2/n^3$.

\emph{End pieces.} On $[0,T_{(1)}]$, $|e(t)|=|f(T_{(1)})-f(t)|\le \kappa(T_{(1)}-t)$, because $\Lip(f)\le \kappa$ (Proposition~\ref{prop:variation}). The contribution of this piece to $\norm e_2^2$ is at most $\int_0^{T_{(1)}}\kappa^2(T_{(1)}-t)^2\frac{\dd t}2=\kappa^2T_{(1)}^3/6$. Here $T_{(1)}/2$ is the minimum of $n$ independent uniform variables on $[0,1]$, so $T_{(1)}/2\sim\mathrm{Beta}(1,n)$ and $\E T_{(1)}^3=8\cdot\frac{3!\,n!}{(n+3)!}=\frac{48}{(n+1)(n+2)(n+3)}$. The right end piece is symmetric. Together, the two end pieces contribute at most $16\kappa^2/((n+1)(n+2)(n+3))$ to $\E\norm e_2^2$.

\emph{Interior gaps.} Let $\nu=f$ on $(0,2)$, a finite combination of Dirac masses at the breakpoints of $f$, with total variation measure $|\nu|$ and $V=|\nu|((0,2))\le2\kappa$ (Proposition~\ref{prop:variation}). The measure $\nu$ is fixed, since $f$ is fixed, while the gaps are random. Throughout, $t$ denotes an evaluation point and $v$ a breakpoint location. For an interval $I=[x,x+\ell]$, define the Green kernel
\[
  G_I(t,v)=\frac{(\min\{t,v\}-x)(x+\ell-\max\{t,v\})}{\ell}\ \ge0,\qquad t,v\in I .
\]
For a single kink $f=(\cdot-v)_+$ with $v\in I$, the chord of $f$ over $I$ minus $f$ equals $G_I(\cdot,v)$ on $I$: for $t\le v$ the difference is $(t-x)(x+\ell-v)/\ell$, and for $t\ge v$ it is $(v-x)(x+\ell-t)/\ell$. On $I$, $f$ is an affine function plus $\int_{\mathrm{int}\,I}(\cdot-v)_+\,\nu(\dd v)$, and the chord is exact on affine functions. Hence, for $t\in I$,
\[
  e(t)=\int_{\mathrm{int}\,I}G_I(t,v)\,\nu(\dd v).
\]
By Cauchy--Schwarz with respect to the finite measure $|\nu|$,
\[
  e(t)^2\le|\nu|(I)\int_{\mathrm{int}\,I}G_I(t,v)^2\,|\nu|(\dd v)\le V\int_{\mathrm{int}\,I}G_I(t,v)^2\,|\nu|(\dd v).
\]
Integrate over $t\in I$ against $\mu$ and exchange the order of integration (all terms are nonnegative). If $v$ lies at distances $\alpha$ and $\beta=\ell-\alpha$ from the endpoints of $I$, then
\[
  \int_IG_I(t,v)^2\,\frac{\dd t}2=\frac12\Big(\frac{\alpha^3\beta^2}{3\ell^2}+\frac{\alpha^2\beta^3}{3\ell^2}\Big)=\frac{\alpha^2\beta^2}{6\ell}\le\frac{\ell^3}{96},
\]
because $\alpha\beta\le\ell^2/4$. For $v\in(0,2)$, let $\ell(v)$ be the length of the interior gap containing $v$, and set $\ell(v)=0$ if $v<T_{(1)}$ or $v>T_{(n)}$. Breakpoints coincide with design points with probability zero. Summing over interior gaps,
\[
  \sum_{I\ \rm interior}\int_Ie^2\dd\mu\le\frac V{96}\int_{(0,2)}\ell(v)^3\,|\nu|(\dd v).
\]
Take expectations over the design and use Fubini's theorem; $|\nu|$ does not depend on the design. Fix $v\in(0,2)$. Map $[0,2)$ to the circle of length $2$. When $T_{(1)}<v<T_{(n)}$, the interior gap containing $v$ is also the arc between consecutive design points that contains $v$. So $\ell(v)$ is at most the length $\ell_\circ(v)$ of that arc. Rotate the circle so that $v$ sits at $0$, and rescale it to length one. The design becomes $n$ independent uniform points, and $\ell_\circ(v)/2=1-R$, where $R$ is their range. Since $R\sim\mathrm{Beta}(n-1,2)$, we get $\ell_\circ(v)/2\overset d=X\sim\mathrm{Beta}(2,n-1)$ and
\[
  \E\,\ell(v)^3\le\E\,\ell_\circ(v)^3=8\,\E X^3=8\cdot\frac{4!\,n!}{(n+3)!}=\frac{192}{(n+1)(n+2)(n+3)} .
\]
Therefore the interior gaps contribute at most $\frac V{96}\cdot V\cdot\frac{192}{(n+1)(n+2)(n+3)}=\frac{2V^2}{(n+1)(n+2)(n+3)}\le\frac{8\kappa^2}{(n+1)(n+2)(n+3)}$ to $\E\norm e_2^2$.

Adding the three contributions gives $\E\norm{\widehat f-f}_2^2\le\sigma^2+24\kappa^2/((n+1)(n+2)(n+3))\le25\kappa^2/n^3$, uniformly over the class. The condition $B\ge \kappa/n$ is used only in the lower bound, and $\sigma^2\le \kappa^2/n^3$ only in the upper bound.
\end{proof}

\subsection{Proof of Proposition~\ref{prop:critical}}
Fix the constants in this order: $c,C,w_0$ from Corollary~\ref{cor:budgets}; then $c_*$; then $c'$ and the final $w_0$. Let $C_U$ be the universal constant of Theorem~\ref{thm:upper}. Assume $w\ge w_0$, $L\ge C\log w$, and $M_*\le n\le M_*^5$, with $M_*=L^2w^2/\log w$. As in the proof of Corollary~\ref{cor:budgets}(c), $d=\lfloor L/(2\log w)\rfloor-1$ satisfies $D_0\le d\le D$ and $M_d\ge c_2M_*$ with $c_2=c_M/256\le1$.

\emph{Sufficiency.} By Corollary~\ref{cor:budgets}(c), in the form $\Risk_n^*(C_A'L,1,1)\ge c\min\{1,M_d/n\}$ from its proof, and since $c_2M_*/n\le1$ for $n\ge M_*$,
\[
  \Risk_n^*(C_A'L,1,1)\ge c\min\Big\{1,\frac{c_2M_*}n\Big\}=c_*\frac{M_*}n,\qquad c_*\defeq c\,c_2 .
\]

\emph{Necessity.} Let $c'\in(0,1)$, to be fixed, and $A\le c'L$. Then $\kappa=(A/L)^L\le c'^L\le c'^2<1$, because $L\ge2$. With $B=\sigma=1$, we have $q=1/n$ and $\Dsig\le2$. The third term of Theorem~\ref{thm:upper}, with the minimum bounded by its second entry, gives
\[
  \Risk_n^*(A,1,1)\le C_U\Big(\frac1n+2^{4/5}\kappa^{2/5}n^{-4/5}\Big)\le\frac{C_U}n\big(1+2c'^{2L/5}n^{1/5}\big)\le\frac{C_U}n\big(1+2c'^{4/5}M_*\big),
\]
using $n^{1/5}\le M_*$ in the last step. This is where the assumption $n\le M_*^5$ enters. The second-variation term exceeds the parametric rate $1/n$ by the factor $\kappa^{2/5}n^{1/5}$, which grows with $n$. So for any fixed $A>0$ the upper bound eventually exceeds $c_*M_*/n$, and some restriction on $n$ is necessary for this argument. Now choose $c'$ with $2C_Uc'^{4/5}\le c_*/4$; this depends only on $C_U$ and $c_*$. Enlarge $w_0$ so that $w_0^2\ge4C_U/c_*$. Since $L\ge C\log w\ge\log w$, we have $M_*\ge w^2\log w\ge w^2$ for $w\ge3$, and so $C_U\le c_*M_*/4$. Hence
\[
  \Risk_n^*(A,1,1)\le\frac{c_*M_*}{4n}+\frac{c_*M_*}{4n}<c_*\frac{M_*}n\qquad\text{for all }0<A\le c'L .
\]

\emph{Critical budget.} The classes $\cC_{L,w}(A,1)$ increase with $A$, so $A\mapsto\Risk_n^*(A,1,1)$ is nondecreasing, and $\{A:\Risk_n^*(A,1,1)\ge c_*M_*/n\}$ is an up-set. It contains $C_A'L$ by sufficiency, and contains no $A\le c'L$ by necessity. Hence $c'L\le A_{\rm crit}(c_*)\le C_A'L$. Both statements use the same threshold $c_*M_*/n$. The proposition says nothing about the larger target $L^2w^2\log w/n$.

%% file: figures/spline_packing.tex
% Spline hypercube of Lemma H.2: disjoint tents and their realization.
\begin{tikzpicture}[
  x=1cm, y=1cm,
  nd/.style={circle, draw=black, line width=0.5pt, fill=white, inner sep=0pt, minimum size=5pt},
  gn/.style={circle, draw=accentblue, line width=0.6pt, fill=white, inner sep=0pt, minimum size=5pt},
  gwire/.style={draw=accentblue, line width=0.7pt},
  lab/.style={font=\footnotesize, inner sep=1pt},
]
% ---------------- (a) the tents ----------------
\begin{scope}
  % [0,1] drawn at 5cm per unit; the rest of the circle [1,2) is indicated by a dotted line
  \draw[line width=0.4pt] (0,0) -- (5.25,0);
  \draw[line width=0.4pt, dotted] (5.25,0) -- (5.8,0);
  \foreach \x/\t in {0/0, 1.25/\tfrac14, 3.75/\tfrac34, 5/1}{
    \draw[line width=0.4pt] (\x,0) -- (\x,-0.08);
    \node[lab, below=2pt] at (\x,0) {$\t$};
  }
  \def\H{0.85}
  % k = 4 tents of base h = 1/(2k) on [1/4,3/4]; omega = (1,0,1,1)
  \draw[line width=0.7pt] (1.25,0) -- (1.5625,\H) -- (1.875,0);
  \draw[line width=0.5pt, dotted, black!60] (1.875,0) -- (2.1875,\H) -- (2.5,0);
  \draw[line width=0.7pt] (2.5,0) -- (2.8125,\H) -- (3.125,0);
  \draw[line width=0.7pt] (3.125,0) -- (3.4375,\H) -- (3.75,0);
  \foreach \j/\x in {1/1.5625,2/2.1875,3/2.8125,4/3.4375}{\node[lab, above=1pt] at (\x,\H) {$\phi_{\j}$};}
  \draw[line width=0.3pt, black!60] (0.9,0) -- (0.9,\H);
  \draw[line width=0.3pt, black!60] (0.85,0) -- (0.95,0) (0.85,\H) -- (0.95,\H);
  \node[lab, left] at (0.9,\H/2) {$\lambda$};
  \draw[line width=0.3pt, black!60] (1.875,-0.3) -- (2.5,-0.3);
  \draw[line width=0.3pt, black!60] (1.875,-0.25) -- (1.875,-0.35) (2.5,-0.25) -- (2.5,-0.35);
  \node[lab, below] at (2.1875,-0.3) {$h$};
  \node[font=\small, anchor=west] at (-0.3,-1.25) {(a)\ \ $f_\omega=\lambda\sum_j\omega_j\phi_j$, \ $\omega=(1,0,1,1)$};
\end{scope}

% ---------------- (b) realization ----------------
\begin{scope}[xshift=7.0cm, yshift=0.25cm]
  \node[nd] (t) at (0,0) {};
  \node[lab, left=1pt] at (t.west) {$t$};
  \node[nd] (s1) at (1.5,0) {};
  \node[gn] (s2) at (2.6,0) {};
  \node[gn] (s3) at (3.7,0) {};
  \node[gn] (sL) at (5.6,0) {};
  \draw[line width=0.5pt] (t) -- (s1);
  \draw[gwire] (s1) -- (s2) -- (s3) -- ++(0.45,0);
  \node[font=\small, text=accentblue] at (4.65,0) {$\cdots$};
  \draw[gwire] (sL) -- ++(-0.45,0);
  \node[lab, right=2pt] at (sL.east) {$f_\omega$};
  \node[lab, align=center, anchor=north] at (0.75,-0.25) {$a^{-(L-1)}f_\omega$\\[2pt] $3k$ atoms};
  \node[lab, text=accentblue, anchor=north] at (3.6,-0.25) {$u\mapsto au$, \ $L-1$ blocks};
  \node[font=\small, anchor=west] at (-0.3,-1.5) {(b)\ \ $a=A/L$; \ cost $\le a^{-(L-1)}16\lambda k^2+(L-1)a$};
\end{scope}
\end{tikzpicture}